\documentclass[11pt]{article}

\usepackage{natbib}      % For bibliography management
\usepackage{graphicx}    % For including graphics
\usepackage{amsmath}     % For math symbols and equations
\usepackage{amsthm}
\usepackage{amsfonts}
\usepackage{amssymb}
\usepackage{hyperref}    % For hyperlinks in the document
\usepackage{enumitem}

\usepackage{subcaption}
\usepackage{caption}
\usepackage{float}
\usepackage{fullpage}
\usepackage{todonotes}

\usepackage{refcount}
\usepackage{authblk}

\newcommand{\N}{\mathbb{N}}
\newcommand{\R}{\mathbb{R}}
\newcommand{\E}{\mathbb{E}}
\newcommand{\D}{\mathbb{D}}
\newcommand{\A}{\mathbb{A}}
\newcommand{\B}{\mathbb{B}}
\newcommand{\Dhp}{\mathbb{H}}
\newcommand{\id}{\mathbb{I}}
\newcommand{\hpi}{\hat{\pi}}
\newcommand{\mcP}{\mathcal{P}}
\newcommand{\mfP}{\mathfrak{P}}
\newcommand{\mcF}{\mathcal{F}}
\newcommand{\mcX}{\mathcal{X}}
\newcommand{\mcA}{\mathcal{A}}
\newcommand{\mcN}{\mathcal{N}}

\newcommand{\dd}{\mathrm{d}}
\newcommand{\diam}{\mathrm{diam}}
\newcommand{\notimplies}{\;\not\!\!\!\implies}

\DeclareMathOperator*{\argmax}{arg\,max}

\newtheorem{definition}{Definition}
\newtheorem{theorem}{Theorem}[section]
\newtheorem{lemma}[theorem]{Lemma}
\newtheorem{corollary}{Corollary}[theorem]
\newtheorem{proposition}[theorem]{Proposition}

\theoremstyle{definition}
\newtheorem{example}[theorem]{Example}

\begin{document}

% Title
\title{Sparse Priors for Efficient Distribution Learning}
% \author{Saumya Goyal \& Dhruv Garg \& Divyan Goyal\\
% \{saumyag \& dhruvgar \& dgoyal\}@andrew.cmu.edu\\
% 10-423/623/723 Generative AI Course Project}

\author[1]{Saumya Goyal}
\author[1]{Barnabás Póczos}
\affil[1]{Machine Learning Department, Carnegie Mellon University}
\date{}
\maketitle

% \linenumbers
% \todo[inline]{We should put this into ICML format. It is not clear now how many more pages needed for the main paper.}
\begin{abstract}
Despite the widespread use and success of generative AI techniques today, theoretical guarantees on learning a distribution supported in $d$ dimensions from $n$ samples degrade as $O(n^{-1/\Theta(d)})$, though shown to be minimax optimal. 
We hypothesize that present bounds are too pessimistic because smoothness assumptions are not enough to capture the structure of distributions that often appear in real applications. Consequently, we introduce the class of sparse priors and define the ``Sparse Dimension" as a measure of sparsity of a prior over the space of all distributions. We show that distribution learning under a $k$-sparse prior achieves a Bayesian risk lower bound of $\Omega(\sqrt{k/n})$ under common distance metrics, and show a matching (up to logarithmic terms asymptotically in $n,k$) upper bound for the TV distance under mild additional assumptions. We show the statistical equivalence of distribution learning and learning to sample in the Bayesian setting so that our results apply to learning to sample as well. While $k$ can still depend on the dimension $d$, or a notion of intrinsic dimension, our results show that learning under an appropriate prior overcomes the curse of dimensionality with respect to the dependence on $n$.
\end{abstract}

\section{Introduction}
% Diffusion models \citep{ddpm_ho_20, song2021scorebased} are one of the most popular and successful models of sampling from an unknown distribution given finitely many samples from the distribution as examples. 
The theoretical minimax optimality, in terms of the sample complexity of samples required for learning to sample from a distribution, have been established for several popular machine learning techniques in recent work. \citet{oko_23_diff} and \citet{cond_diff_fu24} establish the optimality for score-based diffusion under Besov smoothness of the score function and score-based conditional diffusion under Hölder smoothness of the score function respectively. Similar results were established for Generative Adversarial Networks (GANs) by \citet{gan_liang_21}. \citet{est_adv_losses_singh18} establish the minimax equivalence of distribution learning and learning to sample, which allows us to conclude the existence of minimax optimal distribution learners as well.
Theoretical guarantees however, still suffer from a major limitation: the established optimality bounds suffer from the curse of dimensionality. Concretely, given $n$ samples from a target distribution $\pi$ supported on $\R^d$, it is shown that both score-based diffusion models and GANs can learn to sample from a distribution that is at the optimal distance $\Theta(n^{-\frac{c}{d+c'}})$ (upto poly-log terms) from the target distribution as measured by an appropriate Integral Probability Metric (IPM) like the total variation (TV) or Wasserstein distance,
% in terms of TV distance \citep{oko_23_diff, cond_diff_fu24} or a more general  adversarial loss \citep{diffusion_tang_yang_24, gan_liang_21}, 
where $c,c'$ are constants that depend on the smoothness of the score function.% and $d$ is the dimensionality of the space of samples. 

A popular way to circumvent the curse of dimensionality is the manifold hypothesis, under which we assume that the data distribution lies on a lower dimensional submanifold. While \citet{oko_23_diff} and \citet{gan_liang_21} extend their bounds to a version of the manifold hypothesis, 
% consider a linear map that maps points in $\R^{d'}$ for $d'\le d$ to the support of the distribution in $\R^d$, and improve their bounds to $\Theta(n^{-\frac{c}{d'+c'}})$ in TV distance. 
\citet{submanifold_tang_yang_22, diffusion_tang_yang_24} consider the most general submanifold structure comprising an atlas of homeomorphisms from a subset of $\R^{d'}$ for $d'<d$ to the support of $\pi$. \citet{gan_liang_21, oko_23_diff, submanifold_tang_yang_22, diffusion_tang_yang_24} all show a lower bound of $\Omega(n^{-\frac{c}{d'+c'}})$ on the minimax risk defined in Section \ref{sec:ps} and show matching upper-bounds (upto poly-log terms).
% They show similar bounds as above, assuming Hölder smoothness for W1 distance, but not for TV distance. 
% Other work \citep{submanifold_tang_yang_22, diffusion_tang_yang_24} note that image distributions typically lie on a lower dimensional sub-manifold and manage to show the optimality of diffusion under the modified assumptions of sub-manifolds. 
The manifold hypothesis offers significant improvements, and allows us to essentially replace the dimensionality of distributions with intrinsic dimensionality, which is indeed much lower than the true dimensionality for image distributions.
% . The intrinsic dimensionality of image distributions is indeed much lower than the true dimensionality. 
For example, \citet{int_dim_gong_19, int_dim_pope_21} estimate the intrinsic dimensionality of imagenet \citep{imagenet_deng09} to be at most $19$ and $43$ respectively, which is much smaller than the true dimensionality of $>150k$. However it is still too large for 
% While the intrinsic dimensionality of image distributions is indeed much lower than the true dimensionality \citep{int_dim_gong_19, int_dim_pope_21}, it is still too high for 
image distributions to be learnable with the sample sizes we see in practice. To see this, assume the optimistic $d'=19$ and consider the rates of \citet{gan_liang_21}, where for a distribution in a Sobolev class with smoothness $\alpha$, leads to constants $c=\alpha$ and $c'=2\alpha$ when the distance metric is the total variation (TV) distance. Under $\alpha=2$, corresponding to the data-generating distribution admitting smooth second derivatives, one obtains the asymptotic rate of $\Theta(n^{-2/23})$, which means we need more than $10^{11}\times$ samples to observe a $10\times$ improvement in the total variation distance. In order to obtain a modest $\Theta(n^{-1/4})$ bound on distribution learning, we need to assume $\alpha=9.5$, requiring smoothness up to the ninth derivative of the data-generating distribution! On the other hand, we observe much faster rates in practice. \citet{diff_scaling_liang_25} study scaling behavior of the Fréchet Inception Distance (FID) \citep{FID_Heusel_17}, a common pseudo-metric for Generative AI, with compute (which we denote as $\mathrm{comp}$ in this discussion) and data. They observe that optimal training requires the number of samples $n=O(\mathrm{comp}^{0.4319})$, and FID decreases as $O(\mathrm{comp}^{-0.234})$, implying a $O(n^{-0.234/0.4319}) \approx O(n^{-0.54})$ rate of the FID with the number of samples under optimal training (also an assumption made theoretically), much well-behaved than predicted by theory. Similarly, \citet{gen_scaling_henighan_20} observe a $\mathrm{const.}+ O(n^{-0.3})$ relationship of loss (which does not equal to distance between distributions) with the number of samples for images in a $d=64$ dimensional space, and $\mathrm{const.}+ O(n^{-0.26})$ relationship for images in a $d=256$ dimensional space. Here the constant is treated as an un-improvable loss, pertaining to modeling or compute limitations.

In this work, we aim to disambiguate this gap between theory and practice by shifting focus from worst-case minimax bounds to average-case analysis of the Bayesian risk. We hypothesize that smoothness assumptions such as Hölder or Sobolev smoothness on the target distribution are not enough to capture the structure of distributions that often appear in real applications. 
% We prove a Bayes' risk lower bound that doesn't suffer from the curse of dimensionality under a special class of priors that we call ``sparse priors". We further show the existence of an estimator that matches the lower bounds for the TV distance up to logarithmic terms. 
% , implying that generative models don't necessarily have to suffer the curse of dimensionality as posited by theory.
We imagine a realistic prior over the space of distributions to create well-separated, sparse clusters, such that smooth densities are not weighted equally to their smooth perturbations, as explained in Section \ref{sec:sparse_dim}. We introduce the notion of a ``Sparse Dimension" in Definition \ref{def:sparse_dim}, that measures the sparsity of a prior, and goes to $\infty$ as the prior becomes less sparse. 
% We approximate the sparse dimensions of distributions over several real datasets in Section \ref{sec:ex_sparse_dim} and calculate the sparse dimension of the distribution of galaxies in the universe. 
Using the proposed notion of sparse priors, we study the Bayes' risk of distribution learning and learning to sample from distributions. We show in Section \ref{sec:lbs} that a $k$-sparse prior admits a lower bound of $\Omega(\sqrt{k/n})$ on the Bayes' risk defined in Section \ref{sec:ps} under both the total variation and Wasserstein distances, which is a significant improvement over prior worst case bounds. While $k$ can still depend on and be of the order of $d$, our bounds do not suffer from the curse of dimensionality for any finite $k$ in terms of their dependence on the number of samples. In Section \ref{sec:ub}, we show the existence of a distribution estimator that achieves a Bayes' risk of $O(\sqrt{k\log n}/\sqrt{n})$ for a $k$-sparse prior under the total variation distance, thus matching our lower bounds up to logarithmic terms. We show the statistical equivalence of distribution learning and learning to sample in Sections \ref{sec:lbs} and \ref{sec:ub}. We thus define a realistic class of priors that take a step towards explaining the gap between theory and practice of distribution learning and Generative AI.

\subsection{Summary of Key Contributions}
We summarize our key contributions below:
\begin{enumerate}
    \item \textbf{Sparse Dimension:} We introduce the formal problem statement and notation in Section \ref{sec:ps}. We introduce the notion of sparse priors and sparsity in Section \ref{sec:sparse_dim}. We present an intuitive notion for a realistic prior over the space of distributions that creates well-separated clusters. Within each cluster, most distributions have low density in the prior, with few distributions having high density. We formalize this notion to define the Sparse Dimension in Definition \ref{def:sparse_dim}, which measures how sparse a given distribution is. We show in Section \ref{sec:sparse_dim} that the constructions in popular theoretical minimax guarantees on distribution learning correspond to infinitely high sparse dimensions, and hence do not satisfy our bounds. We show how the geometric notion of Assouad dimension of the support of a distribution can guarantee the sparsity of the distribution, even though that is not a necessary condition in Theorem \ref{thm:assouad_sparse}. We further show that $k$-sparsity implies the existence of a high-measure subset of the support that satisfies similar constraints as the Assouad's dimension in Theorem \ref{thm:sparse_support}. We further show motivating examples of calculations of the sparsity of some distributions in Section \ref{sec:ex_sparse_dim}, which can be used for constructing sparse priors.  
    % We additionally estimate the sparse dimension of a uniform prior over the meta-dataset \citep{mds_triantafillou20} and subsets of classes in imagenet\citep{imagenet_deng09} in Section \ref{sec:ex_sparse_dim}, and calculate the sparse dimension of the distribution of galaxies in the universe as a motivating example.
    \item \textbf{Lower bounds using Sparse Dimension:} We proceed to discuss Bayes' optimality of distribution learning and learning to sample from distributions under a sparse prior. We first show in Theorem \ref{thm:sampl_to_learn} that sampling is at least as hard as learning a distribution under any distance (pseudo-)metric, which allows us to focus on information theoretic lower bounds on distribution learning, which imply bounds on learning to sample. This generalizes the minimax reduction shown in \citet{est_adv_losses_singh18} to the Bayesian case. Theorems \ref{thm:tv_lb} and \ref{thm:w1_lb} show that a $k$-sparse prior defined using the total variation and Wasserstein (W1) distance respectively imply lower bound of $\Omega(\sqrt{k/n})$ on the Bayes' risk, where $n$ is the number of samples. This is interesting on several fronts. First, it proves a bound on distribution learning that doesn't suffer from the curse of dimensionality as in prior work,
    % \citet{gan_liang_21, oko_23_diff, submanifold_tang_yang_22, cond_diff_fu24}, 
    and discusses the Bayes' risk instead of restricting to a pessimistic minimax risk. Second, our proof techniques primarily reduce the analysis of the lower bound of the non-parametric estimation problem to a uniform prior over $k$ parameters. Thus similar bounds for other distance metrics not shown in this paper can be derived using techniques inspired from parametric lower bounds for the respective distance metrics. Third, we see that our bounds are especially distinct from the submanifold analysis of 
    % we note that the bounds of \citet{oko_23_diff, cond_diff_fu24} are only valid for the TV and W1 distances, and the bounds of \citet{gan_liang_21, submanifold_tang_yang_22} only work for IPMs under a Sobolev or Hölder class respectively, that interpolate between the W1 and TV distance. Additionally, the bounds of 
    \citet{submanifold_tang_yang_22} that gives vacuous bounds for the TV distance, or the bounds of \citet{gan_liang_21} that are worse for the TV distance than the W1 distance. % On the other hand, our bounds are valid for all IPMs without any additional assumptions.
    \item \textbf{Upper bounds using Sparse Dimension:} We then shift our focus to upper bounds on Bayesian learning of distributions and learning to sample. Theorem \ref{thm:learn_to_sampl} shows that density estimation is at least as hard as learning to sample under any distance metric, so that our upper bounds on distribution learning, which by definition apply to density estimation, also apply to learning to sample.
    Finally, we prove the existence of an estimator that achieves a Bayes risk of $O(\sqrt{k\log n}/\sqrt{n})$ for a $k$-sparse prior under the TV distance in Theorem \ref{thm:upper_bound}. We utilize the framework of \citet{nonparam_bayes_vdV2000} that study posterior contraction around a single target distribution to show bounds on Bayesian learning assuming a true data-generating prior. 
    Our upper bounds assume that the likelihood ratio between any two distributions is bounded by a constant $b_{lr}$, which is not assumed for our lower bounds. Note that the asymptotic bound does not depend on the value of $b_{lr}$, but only on its finiteness. We exploit this in Corollary \ref{cor:ub_no_lr} to show the same upper bound as in Theorem \ref{thm:upper_bound} while assuming a relaxed bound on the likelihood ratio, which increases sub-linearly in the number of samples $n$ with high probability. We show a simple example of a Gaussian location family that satisfies the assumptions of the Corollary. We also discuss that assuming boundedness of the support of distributions implies an upper bound on the W1 distance in terms of the TV distance, allowing us to extend our results to the W1 distance given a $k$-sparse prior under the TV distance. This is similar to the derivation of upper bounds followed by \citet{oko_23_diff}.
    \item \textbf{Asymptotic Comparison of Bounds:} We have thus shown the existence of upper and lower bounds that match up to logarithmic terms. We highlight that our upper and lower bounds are asymptotic in both $n$ and $k$. This does not make our bounds vacuous for small $k$ however, since by definition a $k$-sparse prior is always $k'$-sparse for $k'>k$. We thus make two separate asymptotic statements for clarity. First, our upper and lower bounds match up to logarithmic terms when considering asymptotics in $n$ and $k$. Second,  our upper and lower bounds match up to logarithmic terms when considering asymptotics in $n$ while keeping $k$ constant for any $k$. In the latter, the constants hidden by the asymptotics may not depend similarly on $k$ for the upper and lower bounds. 
    \item \textbf{Conclusion:} We thus establish a non-trivial class of realistic priors that do not suffer from the curse of dimensionality, which is a step towards explaining the gap between theory and practice in the field of distribution learning and generative AI.
    % \item We note from \citet{oko_23_diff, diffusion_tang_yang_24} that Hölder smoothness of the target pdf allows approximating the pdf as a linear combination of $B$-spline basis vectors. Similar approximations hold for Sobolev smoothness as used in \citet{cond_diff_fu24}. We discuss the scenario when the target pdf is exactly a linear combination of basis functions, or more precisely, the target belongs to a mixture of $k+1$ distributions in Section \ref{sec:ub_mix}. We show in Theorem \ref{prop:mix_prop} that a probability simplex (uniform prior over weights that sum to 1) over the weights of this family has a sparse dimension of $k$ under the TV metric, and show a simple sampling procedure that matches the lower bounds from Theorem \ref{thm:main_lb}.
\end{enumerate}

% We note that our visual model matches the universe

% We imagine the space of distributions for real applications occur as well-separated clusters, where each cluster admits a sparse distribution of likelihood over 

% : are theoretical bounds too pessimistic in real applications, or do diffusion models indeed suffer from the curse of dimensionality?

% \todo[inline]{Can we add an informal paragraph about the main contributions of the paper?}

\subsection{Related Work}
\textbf{Prior Results on Minimax Optimality.} Minimax optimality of diffusion models, or generative AI in general have been studied in many works before. \citet{gan_liang_21} discuss the minimax optimality of Generative Adversarial Networks (GANs), where they show that distribution learning for a Sobolev smooth target distribution suffers from a minimax lower bound of the type $\Omega(n^{-\frac{c}{d+c'}})$ under an IPM corresponding to a Sobolev function class, where $c,c'$ depend on the Sobolev classes. The authors then show that these rates are achievable by a non-parametric Sobolev GAN, and show generalization guarantees for a leaky-ReLU GAN using the pseudo-dimension arguments of \citet{nn_pd_bartlett19}. 
The authors also show that the dependence on $d$ in their bounds can be replaced with $d'$, the intrinsic dimensionality, where the intrinsic dimensionality specifcally determines the exponential rate of decay of the eigenvalues of the integral operator corresponding to an RKHS. \citet{est_adv_losses_singh18} generalize the analysis of \citet{gan_liang_21} by considering general adversarial losses (which are only guaranteed to be pseudo-metrics) and other spaces beyond Sobolev spaces. They additionally show the minimax equivalence of density estimation and learning to sample, so that the results established for sampling in many works discussed herein apply directly to density esimation as well.

\citet{oko_23_diff} was one of the first papers to discuss the minimax optimality of score-based diffusion models. They derive similar lower bounds as \citet{gan_liang_21} for Besov smooth target distributions and under the TV and W1 distances. In addition, they show the existence of a neural network that can approximate the score function, and discuss generalization guarantees on learning the neural networks using the techniques of \citet{suzuki2018adaptivity}. They also show a matching upper bound, and extend their results to intrinsic dimensionality where they define a linear map between the intrinsic space of dimension $d'$ and the support of the true distribution in $\R^d$. 

\citet{submanifold_tang_yang_22} consider a much more general setting of intrinsic dimensionality, where they define a submanifold in $\R^{d'}$ that can be mapped to the support of target distribution in $\R^d$ using an atlas of homeomorphisms. They assume Hölder smoothness of the target distribution and show a similar lower bound of $\Omega(n^{-\frac{c}{d'+c'}})$ for a Hölder class of IPMs that interpolate between the W1 and the TV distance. Interestingly, as noted in the paper, the bounds in \citet{submanifold_tang_yang_22} are vacuous for the TV distance, and match the bounds of \citet{gan_liang_21} if the submanifold reduces to the space $\R^d$. The follow-up work by \citet{diffusion_tang_yang_24} proves that score-based diffusion models achieve this rate (upto poly-log factors) using techniques similar to \citet{oko_23_diff}. Similar to \citet{submanifold_tang_yang_22}, the work by \citet{cond_diff_fu24} assumes Hölder smoothness of the target function. However, \citet{cond_diff_fu24} consider a conditional diffusion model, where instead of learning a target distribution over some space $\mathcal{X}$, the goal is to learn a conditional distribution over $\mathcal{X}$ with inputs from $\mathcal{Y}$. Similar to previous results, they also show minimax bounds of the type $\Omega(n^{-\frac{c}{d+c'}})$ and show that score-based diffusion models match these rates upto poly-log terms using the techniques of \citet{oko_23_diff}. 

\textbf{Computational Efficiency of Learning.}
An interesting observation from the above bounds is that both GANs and score-based diffusion models have been shown to be minimax optimal, yet diffusion often performs better in practice. We conjecture this gap arises from the assumption of perfect optimization of the underlying neural network, whether the neural network corresponds to a GAN \citep{gan_liang_21}, or a score-based diffusion model \citep{oko_23_diff, diffusion_tang_yang_24, cond_diff_fu24}. Computational efficiency and optimization error thus play a significant role in the applicability of bounds, and perfect optimization is not always a realistic assumption under computational constraints. While we don't study training dynamics of any generative or distribution learning algorithm, some work in this direction has been done by \citet{diff_alg_dupuis2025} that show algorithm and data-dependent generalization bounds on learning the underlying neural networks for diffusion. We note in general that we focus on statistical efficiency of algorithms in this work, and defer a discussion on computational efficiency to future work.

% A notable limitation of the above works on minimax optimality is the dependence of bounds on the class of IPM, whereas we manage to show bounds that are valid over all IPM. In addition, our bounds consider the Bayes' optimality of the estimator instead of a pessimistic minimax optimality, and show that $k$-sparse priors lead to bounds that don't suffer from the curse of dimensionality. 
% Another notable limitation of these works is assuming perfect optimisation of the underlying neural network, whether the neural network corresponds to a GAN \citep{gan_liang_21}, or a score-based diffusion model \citep{oko_23_diff, diffusion_tang_yang_24, cond_diff_fu24}. While we don't study training dynamics of any generative algorithm, we note that this gap is bridged by \citet{diff_alg_dupuis2025} that show algorithm and data-dependent generalization bounds on learning the underlying neural networks for diffusion.

% In this work we show that a well-formulated $k$-sparse prior can lead to much better guarantees on sample complexity than in literature. However, we do not 
\textbf{Learning under a Non-Parametric Prior.}
In order to show our bounds, we assume that the true prior is perfectly known in this paper, which may not be true realistically. 
We don't
discuss guarantees on learning the prior, or constructing optimization algorithms to account for priors. The framework of \textit{data-driven algorithm design} discussed in \citet{balcan2020data} suggests tuning hyperparameters for algorithms using previously observed tasks, which is an empirical route to approximating Bayes optimal learning algorithms for an unknown prior. Guarantees on learning hyperparameters for linear regression using previously observed, similar tasks were discussed in \citet{enet_2022}, \citet{balcan_23} and \citet{goyal_25}. \citet{goyal2026_lgd} generalize this by giving guarantees on learning hyperparameters for general regression tasks with a convex loss. 

\section{Problem Statement and Notation}\label{sec:ps}
We consider the problem of learning an unknown distribution $\pi$ defined on the space $\mcX\subseteq\R^d$, given $n$ samples $X_1,\ldots, X_n \stackrel{\text{i.i.d.}}{\sim} \pi$ sampled i.i.d. from $\pi$. We additionally assume that $\pi\sim\mcP$ is sampled from a known prior $\mcP$, over the set of distributions $\mathbb{D}$. 
Let $\hpi(X_1,\ldots,X_n)$ be the estimated distribution from $n$ samples. When obvious from context, we will denote $\hpi(X_1,\ldots,X_n)$ by the short-hand $\hpi_n$. For brevity, we will denote the set of $n$ samples by the singular $\mathbf{X}\in\mcX^n$. 
We will also talk about sampling algorithms, that learn to sample from a distribution given samples. We will analogously denote the learned sampling algorithm by $\hat{\mcA}_n$ and the distribution that the sampling algorithm actually samples from will be denoted $\mathrm{dist}(\hat{\mcA}_n)$. Unless specified explicitly, we assume a distance function $\rho(\cdot,\cdot)$ between distributions to be a pseudo-metric, which means that it is non-negative, symmetric and follows the triangle inequality, but $\rho(\pi_1,\pi_2) = 0  \notimplies \pi_1=\pi_2 $. We are interested in the expected distance between $\hpi_n$ and $\pi$ given by $\E_{X_1,\ldots, X_n \stackrel{\text{i.i.d.}}{\sim} \pi}[\rho(\hpi_n, \pi)]$. We will drop the subscript on the expectation when obvious from context. Define the Bayes' risk of estimators of $\pi$ as:
\begin{align}\label{eq:brisk_est}
    R_{Bayes}(\hpi_n;\mcP) = \E_{\pi\sim\mcP}[\E[\rho(\hpi_n, \pi)]].
\end{align}
We will also be investigating Bayes' risk  for the sampling algorithm, which will be given as follows using slight re-use of notation:
\begin{align}\label{eq:brisk_sampl}
    R_{Bayes}(\hat{\mcA}_n;\mcP) = \E_{\pi\sim\mcP}[\E[\rho(\mathrm{dist}({\hat{\mcA}_n}), \pi)]].
\end{align}

Section \ref{sec:lbs} studies lower bounds on the Bayes' risk of estimation over a class of priors $\mfP$ given by $\inf_{\hpi}\sup_{\mcP\in\mfP}R_{Bayes}(\hpi_n;\mcP)$. 
Correspondingly, the Bayes' risk of sampling will be lower bounded by the expression, $\inf_{\hat{\mcA}}\sup_{\mcP\in\mfP}R_{Bayes}(\hat{\mcA};\mcP)$. This is a notable departure from prior work that study the minimax risk defined by $\inf_{\hpi}\sup_{\pi\in\D} \E[\rho(\hpi_n, \pi)]$  and $\inf_{\hat{\mcA}}\sup_{\pi\in\D}\E[\rho(\mathrm{dist}({\hat{\mcA}_n}), \pi)] $ for distribution learning and sampling respectively.

% For a distance metric $\rho(.,.)$ between distributions, we are interested in the expected distance between $\hpi_n$ and $\pi$, given by $\E_{X_1,\ldots, X_n \stackrel{\text{i.i.d.}}{\sim} \pi}[\rho(\hpi_n, \pi)]$. 

\noindent We will denote by $\rho_\mcF$ the Integral Probability Metric (IPM) corresponding to the function class $\mcF$ as defined below. Note that by definition, any IPM is a pseudo-metric.
\begin{definition}[Integral Probability Metric (IPM) \citep{stat_ot_chewi_24}]
    A (pseudo-)metric $\rho_\mcF(.,.)$ between two probability measures $\mu,\nu$ is called an integral probability metric (IPM) if it can be written in the form
    \begin{align*}
        \rho_{\mathcal{F}}(\mu,\nu) = \sup_{f\in \mathcal{F}} \left|\int_\mcX f(X) \dd\mu(X) - \int_\mcX f(X) \dd\nu(X)\right|.
    \end{align*}
\end{definition}
\noindent Several popular distance metrics are IPMs over different classes $\mcF$. For instance, the Wasserstein-1 distance is the IPM over $\mcF_{Lip1}$, the class of all 1-Lipschitz functions, the TV distance is half of the IPM over $\mcF_1$, the class of all 1-bounded functions \citep{submanifold_tang_yang_22}, and Maximum Mean Discrepancy (MMD) \citep{MMD_gretton_2012} is the IPM over functions in an appropriate RKHS $\mathcal{H}$ with function norm $\|f\|_\mathcal{H} \le 1$ \citep{stat_ot_chewi_24}.
% % \noindent Popular distance metrics such as the Wasserstein-1 distance ($\mathcal{F}$ is the set of all 1-Lipschitz functions), MMD and TV distance ($\mathcal{F}$ is the set of all functions bounded by 1) are examples of IPM. 
As a shorthand, we will write $\rho_{W1}$ and $\rho_{TV}$ when referring to the Wasserstein-1 and TV distance respectively. We will also denote the KL-divergence (which is not a pseudo-metric) between distributions $\mu,\nu$ as $D_{KL}(\mu\|\nu)$.

Lastly, we will define the ball of radius $r$ around distribution $\nu$ for a distance (pseudo-)metric $\rho$ as $B_\rho(r,\nu) = \{\mu \in \mathbb{D}:\rho(\mu,\nu)\le r\}$. For a distribution $\mcP$ over $\mathbb{D}$, we will denote the total probability of the ball by $\mcP(B_\rho(r,\nu))$.

\section{Sparse Dimension}\label{sec:sparse_dim}
As we mention previously, we hypothesize that theoretical bounds on distribution learning tend to be  pessimistic because smoothness assumptions are not enough to capture the structure of distributions that often appear in real applications. As an example, we show the pdfs of two distributions in Figure \ref{fig:dist_comp} that are both heuristically ``smooth'' and ``close'' to each other. These distributions would produce similar samples, and as such, we might not want to distinguish between them. Such constructions are common in literature, such as the minimax construction of \citet{submanifold_tang_yang_22} that consists of a uniform distribution over a ball, with small perturbations.

% As an example, we consider the two distributions shown in Figure \ref{fig:dist_comp}. While both the distributions are heuristically ``smooth" and ``close" to each other, we are unlikely to want to distinguish between samples from both distributions. We hence, don't want to spend compute learning to discriminate between them.
\begin{figure}
    \centering
    \includegraphics[width=0.4\linewidth]{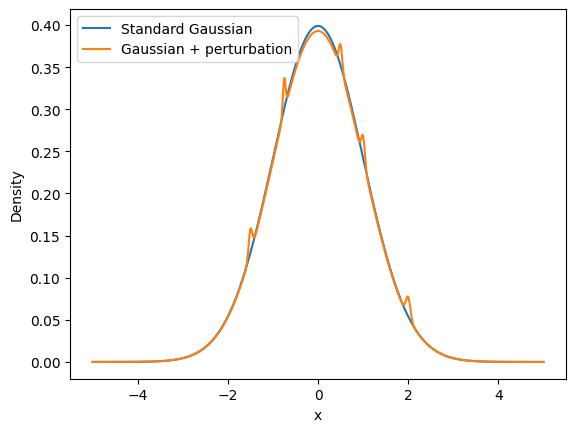}
    \caption{The standard Gaussian illustrated with a small perturbation of it, of the type common information theoretic lower bounds often consider. The perturbations are small enough that it would take an increasing number of samples to distinguish them as the dimensionality increases. In our framework however, we argue that such small perturbations likely receive small prior mass in realistic scenarios, which relaxes the sample complexity requirement for identifying the more likely distribution.}
    \label{fig:dist_comp}
\end{figure}

We propose assuming a prior over the space of all distributions that 
% We introduce a prior over distributions that 
captures the structure of realistic distributions.
Intuitively, such a prior should create well-separated clusters in the space of all distributions. 
Within each cluster, most distributions have low density in the prior, with few distributions having high density. As an example, we consider the distributions over animal images, where we expect the clusters for different animal images like cats and dogs to be well-separated. A prior over the space of animal image distributions should give equal weight to the cat image and the dog image clusters. Inside the dog image cluster, we expect to see a global distribution of all dog images with high density in the prior, but several distributions for different breeds of dogs, or dogs doing different actions, all with a lower density than the global distribution. However a distribution corresponding to dog images with a particular pixel blacked out is probably not of interest to us, and would have very low/zero density.

Thus we intuit that most of the probability mass in realistic distributions is concentrated on a small subset of all possible distributions, with the vast majority of the space being nearly empty. This sparsity extends recursively so that only a few sub-distributions within each cluster—e.g., cats, dogs, or specific dog breeds—capture the essential structure, while all other sub-distributions have negligible mass. 
% This intuitive structure is observed for example, in the universe with clustering of objects with lighter objects orbiting larger ones. 
% We explore this intuition further in Section \ref{sec:ex_sparse_dim}. 
The analysis of \citet{atlas_models_horwitz25} for world models, where the authors suggest plotting an atlas of all world models, suggests a similar structure to the one described above if we plot the world models based on their similarity to each other. Before we define the sparse dimension we note the following definition of the diameter of a set.

\begin{definition}[Diameter]
    The diameter of a set $S$ equipped with a distance (pseudo-)metric $\rho$, denoted $\diam(S)$, is defined such that $\diam(S)=\sup_{a_1,a_2\in S} \rho(a_1,a_2)$.
\end{definition}

We now define the sparse dimension that formalizes our intuition that for a sparse distribution, a big fraction of the probability of a set should lie inside a relatively small subset. 

\begin{definition}[Sparse Dimension]\label{def:sparse_dim}
    Given a distribution $\mcP$ on the space $\D$ and a distance (pseudo-)metric $\rho$, we say that $\mcP$ is sparse with sparse dimension $k>0$ %if $k $ is the smallest number 
    if there exists a $c>0$ 
    such that for every set $\mathbb{A}\subseteq \D$ and $\epsilon \in (0,1]$, there exists a subset $\mathbb{B}\subseteq \mathbb{A}$ that satisfies $\mcP(\B)\ge \epsilon\mcP(\A)$ and $\diam(\B)\le c\epsilon^{1/k}\diam(\A)$.
\end{definition}

Intuitively, we call a distribution sparse if every set in the support of the distribution contains a small, dense subset. The sparse dimension is a measure of the sparsity, where a lower sparse dimension implies a sparser distribution. Note that by definition, a $k$-sparse distribution is immediately $k'$-sparse for any $k'>k$. That is, the definition does not require $k$ to be the smallest number that satisfies the sparsity condition. We show below how certain properties of the support $\D$ of the distribution $\mcP$ can imply the $k$-sparsity of $\mcP$, although these aren't necessary conditions. We first introduce the notion of covering number and Assouad dimension of a set.

We define the $r$-covering number of a space as the number of points in the smallest set required to ensure that all points in the space are at a distance at most $r$ from some element on the set.

\begin{definition}[Covering number \citep{ML_theory_Shwartz_14}]
    For a set $\D$ with a distance (pseudo-)metric $\rho$, consider a set of points $\A$ such that for any element $\mu\in\D$, $\exists \nu \in \A$ such that $\rho(\mu,\nu)\le r$. Then $\A$ is called an $r$-cover of $\D$. We define $N(r,\D)$, the $r$-covering number of $\D$, as the cardinality of the smallest $r$-cover of $\D$. Thus, if the $r$-cover $\A$ is the smallest possible, $|\A| = N(r,\D)$.
\end{definition}

We then define the Assouad dimension, which bounds the $r$-covering number for any subset of diameter $R$.

\begin{definition}[Assouad dimension \citep{assouad_fraser_2020}]\label{def:assouad}
    Consider a non-empty set $\D$ with a distance (pseudo-)metric $\rho$. The Assouad dimension of $\D$, denoted $\mathrm{dim}_A(\D)$ is defined as:
    \begin{align*}
        \mathrm{dim}_A(\D) = \inf \{k: \text{there exists } C>0 \text{ such that } &\forall\; 0<r<R \text{ and } \mu\in\D,\\ &N(r,B_\rho(R,\mu)\cap\D) \le C(R/r)^{k}\}.
    \end{align*}
\end{definition}

Thus, given a set with Assouad dimension $k$ implies the $r$-covering of any subset of diameter $R$ is proportional to $(R/r)^k$. This allows us to bound the sparse dimension of a distribution over the set as follows.

\begin{theorem}\label{thm:assouad_sparse}
    Consider a set $\D$ equipped with the distance (pseudo-)metric $\rho$ with Assouad dimension $k$. For any $\delta>0$, any distribution on $\D$ is always at most $(k+\delta)$-sparse.
\end{theorem}
\begin{proof}
    It follows from Definition \ref{def:assouad}  that for some $\nu\in \D$, $B(R,\nu)\cap \D$ is covered by at most $C(R/r)^{k+\delta}$ balls of radius $r$ centered at points in $\D$, for some constant $C>0$. Now, for any set $\A\subseteq \D$ of diameter $\diam(\A)=R$, construct an $R$ ball that encompasses all points in $\A$, centered at one of the points in $\A$. For any $\epsilon\in(0,1]$, we construct balls of radius $r=(C\epsilon)^{1/(k+\delta)} R$ such that we need at most $\lfloor 1/\epsilon\rfloor$ such balls to cover the set $\A$, where $\lfloor\alpha\rfloor$ is the greatest integer $\le \alpha$. Note that if $(C\epsilon)^{1/(k+\delta)} > 1$, then $r>R$ and we only need 1 such ball to cover the set $\A$, which still falls inside our $\le \lfloor1/\epsilon\rfloor$ bound. By a pigeonhole argument, the intersection of at least one of these balls with $\A$ has a total probability $\ge \epsilon \mcP(\A)$ with diameter $\le 2(C\epsilon)^{1/(k+\delta)} R$, such that the intersection of this ball with $\A$ satisfies the conditions of Definition \ref{def:sparse_dim} for a constant $c=2C^{1/(k+\delta)}$.
\end{proof}

A few corollaries are immediate, which will be useful to us later. We show that any distribution over a $k$-Ahlfors regular set is always at most $k$-sparse.

\begin{corollary}\label{cor:alfhors}
    Consider a $k$-Ahlfors regular set $\D$ equipped with the distance (pseudo-)metric $\rho$ such that there exists a Borel measure $\mathcal{M}$ with support $\D$ and a constant $C\ge 1$ such that for any $\nu\in\D$ and radius $0<r \le \diam(\D)$, $C^{-1}r^k \le \mathcal{M}(B_\rho(r,\nu)\cap \D) \le Cr^k $. Any distribution on $\D$ is always at most $k$-sparse.
\end{corollary}
\begin{proof}
   From Proposition \ref{prop:ahlfors_cover}, we note that for some $\nu\in \D$, $B(R,\nu)\cap \D$ is covered by at most $C^2(3R/r)^k$ balls of radius $r$ centered at points in $\D$. Thus $\D$ has Assouad dimension $k$, and similar arguments to Theorem \ref{thm:assouad_sparse}, where we can now replace $(k+\delta)$ with $k$ (for a different value of $C$).
\end{proof}

Using a similar technique, we show that a distribution over a finite dimensional space is always sparse given an appropriate distance (pseudo-)metric. We will rely on this result for our construction of hard instances for showing lower bounds in Section \ref{sec:lbs}.

\begin{corollary}\label{cor:d_sparse}
    Any distribution over a subset of $\R^d$, denoted $\D\subseteq \R^d$ equipped with an $l_a$ norm for $a\in[1,\infty]$ as the distance (pseudo-)metric is always at most $d$-sparse.
\end{corollary}
\begin{proof}
    Proceeds exactly the same as Corollary \ref{cor:alfhors} with $C=1$ by invoking Proposition \ref{prop:cover_d_dim} instead of Proposition \ref{prop:ahlfors_cover}.
\end{proof}

We note that $k$-sparsity of a distribution does not necessarily imply boundedness of the covering number of it's support. We show below that $k$-sparsity implies that there exists a high-measure subset of the support with a covering number with a similar expression as the Assouad dimension.

\begin{theorem}\label{thm:sparse_support}
    Assume there exists a $k$-sparse prior $\mcP$ that satisfies Definition \ref{def:sparse_dim} with constant $c$ over the support $\D$ of diameter $R$ under a distance (pseudo-)metric $\rho$. Then for any $\delta\in(0,1)$ and $0<r<cR$, there exists a subset $\B\subseteq\D$ such that $\mcP(\B)\ge \delta\mcP(\D)$ and the covering number of $\B$ is bounded as:
    \begin{align*}
        N(r,\B) \le \left\lceil\frac{\log(1-\delta)}{\log(1-(r/(cR))^k)}\right\rceil \le \left\lceil-\log(1-\delta)\frac{(cR)^k}{r^k}\right\rceil,
    \end{align*}
    where $\lceil\alpha\rceil$ is the least integer $\ge \alpha$, and since $-1/\log(1-x)\le 1/x$ for $x\in(0,1)$. 
\end{theorem}
\begin{proof}
    Let $\epsilon= (r/(cR))^k $ in Definition \ref{def:sparse_dim}, such that there exists a set $\A_1\subseteq\D$ of diameter $\le r$ such that $\mcP(\A_1)\ge \epsilon\mcP(\D)$. % Without loss of generality, we assume here that $\epsilon<1$, since otherwise we can replace $\epsilon$ with $1$ without changing our conclusions. 
    For any $N$, we can similarly find disjoint sets $\{\A_i\}_{i=1}^N$ which satisfy $\A_{i+1}\subseteq \D/(\bigcup_{j=1}^i\A_j)$ such that  $\diam(\A_{i+1})\le (r/R) \diam(\D/(\bigcup_{j=1}^i\A_j)) \le r$ and $\mcP(\A_{i+1})\ge \epsilon \mcP(\D/(\bigcup_{j=1}^i\A_j)) = \epsilon(1-\sum_{j=1}^i\mcP(\A_j))$. Note that the sets $\A_i$ are allowed to be null. Since each of these sets has diameter $\le r$, an $r$-ball centered at any point in the set covers the set. (Alternatively, if the set is null it is vacuously covered by any $r$-ball.) Let $\B = \bigcup_{i=1}^N \A_i$ for $N$ to be chosen later, such that $\B$ is covered by $N$ balls of radius $r$. $\B$ has total measure % While these $N$ balls may not be disjoint, we note that they have disjoint components such that the total measure of these $N$ balls is 
    $\mcP(\bigcup_{i=1}^N\A_i) = \sum_i \mcP(\A_i)$. We simplify this as follows:
    \begin{align*}
        \mcP(\bigcup_{i=1}^N\A_i) &= \sum_{i=1}^N \mcP(\A_i) \ge \sum_{i=1}^{N-1} \mcP(\A_i) + \epsilon(1-\sum_{i=1}^{N-1} \mcP(\A_i)) = \epsilon + (1-\epsilon)\sum_{i=1}^{N-1} \mcP(\A_i)\\
        &\ge \epsilon + \epsilon(1-\epsilon) + (1-\epsilon)^2\sum_{i=1}^{N-2} \mcP(\A_i) \ge \epsilon\sum_{i=1}^{N-1}(1-\epsilon)^{i-1} + (1-\epsilon)^{N-1}\mcP(\A_1)\\
        &\ge \epsilon\sum_{i=1}^{N}(1-\epsilon)^{i-1} = 1-(1-\epsilon)^{N}.
    \end{align*}

    Let $N = \left\lceil\frac{\log(1-\delta)}{\log(1-(r/(cR))^k)}\right\rceil$ so that $\mcP(\B) \ge \delta\mcP(\D)$ as required.
\end{proof}

We now discuss some examples on Sparse Priors.

\begin{example}
    % As an example of a sparse distribution, we note that any distribution over a $d$-dimensional Euclidean space is always at most $d$-sparse. To see this, we first note that the total volume of any $r$-ball in a $d$-dimensional Euclidean space is $\propto r^d$. We can thus fit at most $\frac{R^d}{(R+r/2)^d}$
    % Thus any ball of radius $R$ can be covered by at most $(3R/r)^d$ balls of radius $r$ (Proposition \ref{prop:cover_d_dim}). Now, for any set $\A$ of diameter $\diam(\A)$, construct an $R=\diam(\A)$ ball that encompasses all points in $\A$. For any $\epsilon\in(0,1]$, we construct balls of radius $r=(3\epsilon)^{1/d} R$ such that we need at most $1/\epsilon$ such balls to cover the set $\A$. By a pigeonhole argument, we note that at least one of these balls has a total probability $\ge \epsilon$. Such a ball has diameter $\le 2(3\epsilon)^{1/d} R$, which satisfies the conditions of Definition \ref{def:sparse_dim} for a constant $c=2(3)^{1/d}$. 

    As an example of a sparse prior, we show an instantiation of the above result. Consider the set of Gaussian distributions with unit covariance $\D=\{\mathcal{N}(\theta,I_d):\theta\in\R^d, \|\theta\|\le 1\}$. The Wasserstein-1 metric between two Gaussians is directly equal to the Euclidean distance between means as given by, $\rho_{W1}(\mathcal{N}(\theta_1,I_d), \mathcal{N}(\theta_2,I_d)) = \|\theta_1-\theta_2\|$. Thus, any distribution over the means $\theta$ results in an at most $d$-sparse distribution over all Gaussian distributions with unit covariance in $d$ dimensions. Bounds of the form $\Theta(\sqrt{d/n})$ from Sections \ref{sec:lbs} and \ref{sec:ub} are trivially known for a uniform distribution over the means, and satisfied by the maximum likelihood estimator \citep{all_of_stats}.
\end{example}

\begin{example}
    As an example of a prior that is not sparse, we consider the treatment of \citet{submanifold_tang_yang_22}. We set $\gamma\rightarrow0^+$ in Lemma 2 of the paper to see that for any $b>0$, $n>0$ and $\alpha$ that depends on smoothness assumptions, there exist $H\ge \exp(b^dn^{\frac{d}{2\alpha+d}})$ distributions $\mu_1,\ldots, \mu_H$ with a Hölder smooth pdf such that:
    \begin{enumerate}
        \item $D_{KL}(\mu_h\|\mu_l)\le c_1b^{-2\alpha}n^{-\frac{2\alpha}{2\alpha + d}}$ for any $h\ne l$. That is, $\rho_{TV}(\mu_h,\mu_l)\le \sqrt{c_1}b^{-\alpha}n^{-\frac{\alpha}{2\alpha + d}}$ for any $h\ne l$, using the Pinsker's inequality $\rho_{TV}(\mu,\nu)\le \sqrt{D_{KL}(\mu\|\nu)/2}$.
        \item $\rho_{TV}(\mu_h,\mu_l)\ge c_2b^{-(\alpha+d)}n^{-\frac{\alpha}{2\alpha+d}}$.
    \end{enumerate}
    Consider exactly $H = \exp(b^dn^{\frac{d}{2\alpha+d}})$ of these distributions which form a set $\A$ of diameter $R\le \sqrt{c_1}b^{-\alpha}n^{-\frac{\alpha}{2\alpha + d}}$ in the TV space from point 1. Further, there can be only 1 of these distributions in any space of diameter $<r=c_2b^{-(\alpha+d)}n^{-\frac{\alpha}{2\alpha+d}}$. A uniform prior over these $H$ distributions thus allows the maximum measure over any space of diameter $<r$ to be at most $1/H$. Let $\epsilon=2/H$ in Definition \ref{def:sparse_dim}, such that $k$-sparsity implies the existence of set $\B$ with measure $\mcP(\B)\ge 2/H$ with diameter $\diam(\B)\le c(2/H)^{1/k}R $, for some constant $c$. Since any space of diameter $< r$ has measure $\le 1/H$, $\diam(\B) \ge r \implies r\le c(2/H)^{1/k}R \implies H\le 2c^kb^{dk}$. While $H$ can increase arbitrarily with $n$, the RHS $2c^kb^{dk}$ cannot, and thus no such $k$ and $c$ exist. Hence this prior is not sparse. We re-direct the attention of the reader to Figure \ref{fig:dist_comp} that shows a construction with small perturbations where we expect both distributions to not receive equal weight in realistic scenarios corresponding to sparse priors. We expect a similar analysis for the treatment of other works such as \citet{gan_liang_21, oko_23_diff} would follow to show that the assumptions necessitate a non-sparse prior.
\end{example}

% \begin{remark}
%     Note that sparsity of a distribution does not directly imply finiteness of the covering number as the examples above may suggest. 
% \end{remark}

% \todo[inline]{Can we create a figure to illustrate this? Is there a simple way to explain the intuition behind this? Or provide some very simple examples where the sparse dimension is exactly $k$?. Maybe explain what happens with $d$-dim uniform or Gaussian distributions?}

Note that our definition of sparse dimension controls the rate at which the total measure of any set increases with respect to its diameter. This notion is similar to the notion of doubling dimension as defined in \citet{ddim_poczos_13} and previously in \citet{knn_intrinsic_kpotufe2011}, where the authors define a distribution $\mcP$ over $\D$ equipped with a distance metric $\rho$ to have a doubling dimension $k$ if there exists a constant $C$ if for every point $\nu\in\D$ and $0<r<R$,
\begin{align*}
    \frac{\mcP(B_\rho(\nu,R))}{\mcP(B_\rho(\nu,r))} \le C\left(\frac{R}{r}\right)^k.
\end{align*}

While the definition of doubling dimension implies that every $R$-ball has a ``dense'' $r$-ball core, our definition is more relaxed and allows the ``dense'' $r$-ball to be located anywhere within the $R$-ball instead of necessarily at the center. This allows our definition to be a measure of sparsity as opposed to doubling dimension which is intuitively a measure of uniformity. To see this, imagine a very sparse measure such that $\exists \nu,r,R \quad \mcP(B_\rho(\nu,r)) = 0$ but $\mcP(B_\rho(\nu,R)) \neq 0$. While the doubling dimension of such a space would be $\infty$, the sparse dimension can still be finite. 

On the other hand however, our definition is also more restrictive as we control the growth in every set $\A$ whereas doubling dimension only controls the rate for perfect balls. As we note later in Section \ref{sec:ub}, this restriction allows us to prove upper bounds that match the $\Omega(\sqrt{k/n})$ lower bounds of Section \ref{sec:lbs} up to logarithmic terms. 

We denote the set of all $k$-sparse priors on a space $\D$ with distance (pseudo-)metric $\rho$ as $\mfP(k,\rho, \D)$, or as a shorthand $\mfP$ when clear from context.

\subsection{Example Computations of Sparse Dimension} \label{sec:ex_sparse_dim}
We show some more motivating examples showing sparsity of some distributions. Our first example is motivated by the intuition for the structure of sparse dimension as presented in the beginning of this Section. This intuitive structure is observed in the universe with clustering of objects with lighter objects orbiting larger ones. 

\begin{example}[Sparse Dimension of the Universe]
    We imagine the space of all distributions of interest as sparse and clustered, similar to the distribution of dark matter haloes in the Universe. We hence calculate the sparse dimension of the Universe as a model for the sparse dimension of distributions of interest. The distribution of dark matter haloes is usually modelled using a two-point correlation function $\xi(.)$, that models the excess probability of a halo at a distance $r$ from a given halo. That is, in a space of dimension $d$, the probability of a halo at distance atmost $r$ from a halo will be given by:
\begin{align*}
    \mcP_\xi(B(r,\nu)) \propto \int_{r'=0}^r (\xi(r') + 1)r^{\prime (d-1)}\dd r'.
\end{align*}
For the 3D universe, over a range of distances $0.1h^{-1}Mpc \le r\le 10h^{-1}Mpc$ and for $r_0 = 5h^{-1}Mpc$, common surverys approximate the two-point correlation function using a power law given by $\xi(r)\propto (r/r_0)^{-\alpha}$ for $\alpha\approx1.8$ \citep{2pcr_Zehavi_2005, 2pt_corr_davis83}. In this discussion we allow $0h^{-1}Mpc \le r\le 10h^{-1}Mpc$ to permit us to calculate the probability of an $r$-ball for all $r$ in the range as is required by Definition \ref{def:sparse_dim}.
Since the distribution of haloes is isotropic, denote the total probability of any $r$-ball to be $V(r)$ which is given as $V(r) = \mcP_\xi(B(r,\nu)) \propto r^d(1 + \frac{d}{d-\alpha}(r/r_0)^{-\alpha})$. To calculate the sparse dimension, consider a set $\A$ of diameter $R$. Since $\A$ lies within a ball of radius $R\le 10h^{-1}Mpc$, $\mcP(\A) \le V(R)$. Further, $\A$ can be covered by at most $V(R+r/2)/V(r/2)$ balls of radius $r/2$ (similar to Proposition \ref{prop:cover_d_dim}). By a simple pigeonhole argument, the intersection of at least 1 of these $r/2$-balls with $\A$ has weight $\ge \frac{V(r/2)}{V(R+r/2)}\mcP(\A)$. Thus, if we pick an $r$ such that for $\epsilon>0$, $\epsilon \le \frac{V(r/2)}{V(R+r/2)}$, then we can show the existence of a set $\B$ which satsifies $\mcP(\B)\ge\epsilon\mcP(\A)$ with diameter $r$. Note the following bound,
\begin{align*}
    \frac{V(r/2)}{V(R+r/2)} &=\frac{r^d}{(2R+r)^d}\left(\frac{1+\left(\frac{d}{d-\alpha}\right)\left(\frac{r}{2r_0}\right)^{-\alpha}}{1+\left(\frac{d}{d-\alpha}\right)\left(\frac{2R+r}{2r_0}\right)^{-\alpha}}\right)\\
    & = \frac{r^{d-\alpha}}{(2R+r)^{d-\alpha}}\left(\frac{\left(\frac{r}{2r_0}\right)^{\alpha}+\left(\frac{d}{d-\alpha}\right)}{\left(\frac{2R+r}{2r_0}\right)^{\alpha}+\left(\frac{d}{d-\alpha}\right)}\right)\ge \frac{r^{1.2}}{(3R)^{1.2}}\left(\frac{2.5}{3^{1.8}+2.5}\right) \ge 0.068\frac{r^{1.2}}{R^{1.2}}.
\end{align*}
Where we substitute $d=3, \alpha=1.8$ so that $d/(d-\alpha) = 2.5$, lower bound $r/(2r_0)\ge 0$ and upper bound $(2R+r)\le 3R$ and $(2R+r)/(2r_0)\le (3R/2r_0) \le 3$. Thus for $r=(\epsilon/0.068)^{1/1.2}R$, we show a set $\B$ with diameter $r$ and probability $\ge \epsilon\mcP(\A)$ for any $\epsilon>0$. Thus the Universe follows a sparse prior of sparse dimension $1.2$ in the ambient range of distances.
\end{example}

We now show the sparse dimension of the Cantor distribution which can be used to construct distributions with non-integer sparsity.

\begin{example}[Sparse Dimension of the Cantor Distribution]
    To define the Cantor distribution, we need to define the Cantor set. Let $a\in(0,1/2)$ be some constant, $C_0^a = [0,1]$, and define sets $C_t^a$ recursively as: $C_{t+1}^a = \bigcup_{[i_1,i_2]\in C_t^a} ([i_1,(i_2-i_1)*a + i_1] \cup [i_2 - (i_2-i_1)*a, i_2])$. For $a=1/3$, which is the common choice, $C_1^a = [0,1/3] \cup [2/3,1] $, $C_2^a = [0,1/9] \cup [2/9,1/3] \cup [2/3,7/9] \cup [8/9,1] $, and so on. The Cantor set is the intersection of these sets such that $C^a = \cap_{t=0}^\infty C_t^a $. % Note that the Cantor set is well-defined such that for any $a\in(0,1/2)$, each $x\in[0,1]$ is either in $C_t^a$ or it is not.
The Cantor distribution, which we denote $\pi_{C^a}$ is defined such that it assigns a total probability of $2^{-t}$ to each of the $2^t$ intervals of $C_t^a$. We claim that this distribution is $(-\log_a2)$-sparse. 

To show this, consider a set $\A$ with diameter $\diam(\A)=R\le 1$. Consider the interval $[-R/2+x,x+R/2]\supseteq \A$ for some $x\in[0,1]$. Let $t_1  = \lfloor\log_a R\rfloor$, where $\lfloor\alpha\rfloor$ is the greatest integer $\le \alpha$. Then, each interval of $C_{t_1}^a$ has length at least $R$ and $[-R/2+x,x+R/2]$ intersects with at most $2$ continuous intervals of $C_{t_1}^a$. Similarly, let $t_2 = \lceil\log_a r\rceil$ for any $r<R$, where $\lceil\alpha\rceil$ is the least integer $\ge \alpha$, such that each interval of $C_{t_2}^a$ has size at most $r$. Thus, each interval of $C_{t_2}^a$ can be covered by an interval of diameter $r$, and each interval of $C_{t_1}^a$ contains $2^{t_2-t_1}$ intervals of $C_{t_2}^a$. Thus, $2\cdot 2^{t_2-t_1} \le 2^{3+\log_a(r/R)}$ intervals of length $r$ are enough to cover an interval of length $R$, and hence $\A$. By a pigeonhole argument, the intersection of at least one of these intervals with $\A$ must have probability $\ge 2^{-3-\log_a(r/R)}\mcP(\A)$. Now by Definition \ref{def:sparse_dim}, we want to find $c,k$ such that for any $\epsilon\in(0,1]$ there is a subset $\B$ of $\A$ such that $\mcP(\B)\ge \epsilon\mcP(\A) $ and $\diam(\B)\le c\epsilon^{1/k}\diam(\A) $. We select the subset $\B$ as the highest density subset above so that $\mcP(\B) \ge 2^{-3-\log_a(r/R)}\mcP(\A) $ and $\diam(\B)=r$. Now, $\epsilon = 2^{-3-\log_a(r/R)} \implies r/R = (8\epsilon)^{-\log_2 a}$. Thus, the Cantor distribution is $k=(-\log_a 2)$-sparse with constant $c=a^{-3}$. % To show that it is exactly $(-\log_a 2)$-sparse, consider $\A = [0,1]$ and $\epsilon = 1/2$ such that the diameter of any set with probability $\ge 1/2$ is at least $a=(1/2)^{-\log_2 a}$.
\end{example}

\section{Lower Bounds using Sparse Dimension}\label{sec:lbs}
The main technical results of this Section are presented in Theorems \ref{thm:tv_lb} and \ref{thm:w1_lb} that prove dimension-independent lower bounds of $\Omega(\sqrt{k/n})$ on the Bayes' risk of learning a distribution in a $d$-dimensional space under a $k$-sparse prior. 
It is, however, not obvious that sampling is reducible to learning a distribution. 
% \todo[inline]{We should formally define what "sampling" is, and what it means that sampling is suffering from the curse of dimensionality or from some other bounds.} 
That is, it is not obvious that the Bayes' risk of sampling as defined in Equation \ref{eq:brisk_sampl} should suffer the same lower bounds that learning a distribution does. In the following Theorem, we show that a sampling algorithm can be used to devise a procedure to get low loss on any IPM, which allows the reduction of sampling to distribution learning. While similar results have been shown in \citet{est_adv_losses_singh18} for the minimax risk, we prove the result for the Bayes' risk. 

% \todo[inline]{Before we make a statement about the Bayes risk, we shoudl define the Bayes risk}
\begin{theorem}\label{thm:sampl_to_learn}
    For a target distribution $\pi$, consider a sampling algorithm $\hat{\mcA}_n$ that samples from the distribution $\mathrm{dist}(\hat{\mcA}_n)$. The smallest Bayes' risk of any sampling algorithm measured under any distance (pseudo-)metric $\rho$ is lower bounded by the smallest Bayes' risk of any estimator of $\pi$. Formally,
    \begin{align*}
        \inf_{\hat{\mcA}_n} \sup_{\mcP\in \mfP} \E_{\pi\sim\mcP}[\E_{\mathbf{X}\sim \pi^{\otimes n}}[\rho(\mathrm{dist}(\hat{\mcA}_n), \pi)]] \ge \inf_{\hat{\pi}} \sup_{\mcP\in \mfP} \E_{\pi\sim\mcP}[\E_{\mathbf{X}\sim \pi^{\otimes n}}[\rho(\hat{\pi}(\mathbf{X}), \pi)]].
    \end{align*}
\end{theorem}
\begin{proof}
    Consider the distribution estimator $\hpi_n = \mathrm{dist}(\hat{\mcA}_n)$ that achieves the same Bayes' risk bound on distribution learning that  $\hat{\mcA}_n$ achieves on sampling. The inequality follows since there can be a more sample efficient distribution estimator that doesn't require learning to sample.
\end{proof}

The above theorem allows us to study lower bounds on sampling algorithms using standard information theory. 
Below we state a dimension-independent $\Omega(\sqrt{k/n})$  lower bound for the total variation distance, which is proved in detail in Appendix \ref{sec:proofs_of_lbs}. We use a family of discrete distributions to demonstrate our hard instance, primarily because TV computation is easy for such a family. Consider elements $z_1,\ldots, z_{k} \subseteq \mcX$ for some integer $k$ to be unequal elements of $\mcX$, and a discrete distribution over these elements given by the probability vector $\theta = (\theta_1,\ldots, \theta_{k}) $. Then, the TV distance between any two such discrete distributions is simplified as follows:
\begin{align}
    \rho_{TV}(\pi_{\theta},\pi_{\theta'}) &= \frac{1}{2}\sum |\pi_{\theta}(z_i)-\pi_{\theta'}(z_i)| = \frac{1}{2}\sum_{i=1}^{k} |\theta_i-\theta_i'|,\label{eq:disc_dist_tv}
\end{align}
which is proportional to the $l1$ norm between $\theta,\theta'$. Any distribution over the family of discrete distributions is a distribution over the $k$-simplex of $\theta$, which is a $k$-dimensional space equipped with the $l1$ norm. By Corollary \ref{cor:d_sparse}, it follows that a discrete distribution over $k$ elements is at most $k$-sparse. Consequently, in the following result, we assume $|\mcX| \ge k $.

\begin{theorem}\label{thm:tv_lb}
    Let $\D$ be the set of all distributions over $\mcX\subseteq\R^d$, and let $\mfP$ be shorthand for $\mfP(k,\rho_{TV},\D)$, be the set of all $k$-sparse priors over $\D$ using the total variation distance. If $|\mcX|\ge k$, the Bayes' risk for any distribution estimator $\hat{\pi}_n$ is lower bounded independent of $d$ as below:
    \begin{align*}
        \inf_{\hat{\pi}}\sup_{\mcP\in \mfP} R_{Bayes}(\hat{\pi};\mcP) = \inf_{\hat{\pi}}\sup_{\mcP\in \mfP} \E_{\pi\sim\mcP}\left[\E_{\mathbf{X}\sim \pi^{\otimes n}}[\rho_{TV}(\hat{\pi}(\mathbf{X}), \pi)]\right] = \Omega(\sqrt{k/n}).
    \end{align*}
\end{theorem}
\begin{proof}
    In order to show Bayes risk lower bounds over all $k$-sparse  priors, we construct a $k$-sparse prior that achieves the desired lower bound. As stated above, our construction is simply a family of discrete distributions over $\ge k$ elements of $\mcX$. Such a distribution has at most $k$ free parameters, and hence any prior is $k$-sparse by Equation \ref{eq:disc_dist_tv} and Corollary \ref{cor:d_sparse}. Constructing a hard prior over the set of discrete distributions matches popular constructions for showing parametric lower bounds for the TV distance. We defer the reader to Theorem \ref{thm:tv_lb_app} in Appendix \ref{sec:proofs_of_lbs} for a detailed proof.
\end{proof}

We show a similar result as above for the Wasserstein-1 (W1) distance. We will use a similar construction of a family of discrete distributions, and will hence again require that $|\mcX|\ge k$.

\begin{theorem}\label{thm:w1_lb}
    Let $\D$ be the set of all distributions over a $d$-dimensional space, and let $\mfP$ be shorthand for $\mfP(k,\rho_{W1},\D)$, the set of all $k$-sparse priors over $\D$ using the Wasserstein-1 distance. For some constant $r_{min}$ independent of $k$ and $d$, assume $\exists z_1,\ldots, z_{\lfloor k\rfloor} \in \mcX$ such that $\min_{i\neq j} \|z_i-z_j\| \ge r_{min}$. Here $\lfloor k\rfloor$ is the greatest integer $\le k$. Then, the Bayes risk for any distribution estimator $\hat{\pi}_n$ is lower bounded independent of $d$ as below:
    \begin{align*}
        \inf_{\hat{\pi}}\sup_{\mcP\in \mfP} R_{Bayes}(\hat{\pi};\mcP) = \inf_{\hat{\pi}}\sup_{\mcP\in \mfP} \E_{\pi\sim\mcP}\left[\E_{\mathbf{X}\sim \pi^{\otimes n}}[\rho_{W1}(\hat{\pi}(\mathbf{X}), \pi)]\right] = \Omega(\sqrt{k/n}).
    \end{align*}
\end{theorem}
\begin{proof}
    We will use the same construction of a uniform prior over a family of discrete distributions as the one for Theorem \ref{thm:tv_lb}. We first show that the Wasserstein distance is bounded on both sides by the TV distance, which allows us to conclude $k$-sparsity of the chosen prior. We will then argue that since TV distance lower bounds W1 distance by a constant factor, the hard instance for TV estimation is also a hard instance for W1 estimation.
    
    Notably, for an appropriately chosen $k'\le \lfloor k\rfloor$, let $z_1,\ldots, z_{k'}$ be the support of the family of discrete distributions that are elements of $\mcX$ guaranteed by the Theorem statement. Then, $\min_{i\neq j} \|z_i-z_j\|\ge r_{min} $ and let $\max_{i\neq j} \|z_i-z_j\| = r_{max} $. For $\Pi(\theta,\theta')$ the set of all couplings over $\pi_\theta$ and $\pi_{\theta'}$, the W1 distance is written as: 
    \begin{align*}
        \rho_{W1}(\pi_\theta,\pi_{\theta'}) = \inf_{\pi\in\Pi(\theta,\theta')} \E_{X,Y\sim \pi}[\|X-Y\|]= \inf_{\pi\in\Pi(\theta,\theta')} \E_{X,Y\sim \pi}[\|X-Y\|\cdot \id(X\neq Y) ].
    \end{align*}
    Note that the TV distance is simply, $\rho_{TV}(\pi_\theta,\pi_{\theta'}) = \inf_{\pi\in\Pi(\theta,\theta')} Pr_{X,Y\sim \pi}(X\neq Y) $. Thus, we bound the W1 distance as follows:
    \begin{align*}
        r_{min}\inf_{\pi\in\Pi(\theta,\theta')} Pr_{X,Y\sim \pi}(X\neq Y) \le \inf_{\pi\in\Pi(\theta,\theta')} \E_{X,Y\sim \pi}&[\|X-Y\|\cdot \id(X\neq Y) ]\\
        &\le r_{max}\inf_{\pi\in\Pi(\theta,\theta')} Pr_{X,Y\sim \pi}(X\neq Y),
    \end{align*}
    which implies,
    \begin{align*}
        r_{min} \rho_{TV}(\pi_\theta,\pi_{\theta'})\le \rho_{W1}(\pi_\theta,\pi_{\theta'})\le r_{max} \rho_{TV}(\pi_\theta,\pi_{\theta'}).
    \end{align*}

    Since the TV distance is proportional to the $l1$ distance by Equation \ref{eq:disc_dist_tv}, the above equation implies $\lfloor k\rfloor$-Ahlfors regularity of our space by Lemma \ref{lem:ahlfors_lipschitz} (considering the Lebesgue measure on the $\lfloor k\rfloor$-simplex, which is a $\lfloor k\rfloor$-dimensional space), implying any prior over the space is $\lfloor k\rfloor$-sparse by Corollary \ref{cor:alfhors}, and hence $k$-sparse.

    Now that we have established that any prior over the set of discrete distributions as defined previously is $k$-sparse, we can select the same prior as in the construction of Theorem \ref{thm:tv_lb}. Thus, 
    \begin{align*}
        \inf_{\hat{\pi}}\E_{\pi\sim\mcP}\left[\E_{\mathbf{X}\sim \pi^{\otimes n}}[\rho_{W1}(\hat{\pi}(\mathbf{X}), \pi)]\right] &\ge r_{min} \inf_{\hat{\pi}}\E_{\pi\sim\mcP}\left[\E_{\mathbf{X}\sim \pi^{\otimes n}}[\rho_{TV}(\hat{\pi}(\mathbf{X}), \pi)]\right] = \Omega(\sqrt{k/n}).
    \end{align*}
\end{proof}

\section{Upper Bounds using Sparse Dimension}\label{sec:ub}
The main technical results of this Section are presented in Theorem \ref{thm:upper_bound}, that shows the existence of an estimator that matches our lower bounds from Section \ref{sec:lbs} up to logarithmic terms, and Corollary \ref{cor:ub_no_lr} that shows a relaxation of the constraints of Theorem \ref{thm:upper_bound}. We first show an analogous result to Theorem \ref{thm:sampl_to_learn} that shows that density estimation, a special case of learning a distribution, is reducible to learning to sample from a distribution. This guarantees that the Bayes' risk of learning to sample as defined in Equation \ref{eq:brisk_sampl} follow the same upper bounds that learning a distribution does. The following theorem shows that any algorithm that learns a density can be used to sample from the distribution by rejection sampling. Similar results have been shown in \citet{est_adv_losses_singh18} for the minimax risk, we prove the result for the Bayes' risk using similar techniques.

\begin{theorem}\label{thm:learn_to_sampl}
    For a target distribution $\pi$, consider density estimators given by $\hpi_n$. The smallest Bayes' risk of any density estimator is lower bounded by the smallest Bayes' risk of any sampling algorithm measured under any distance (pseudo-)metric $\rho$. Formally,
    \begin{align*}
        \inf_{\hat{\pi}} \sup_{\mcP\in \mfP} \E_{\pi\sim\mcP}[\E_{\mathbf{X}\sim \pi^{\otimes n}}[\rho(\hat{\pi}(\mathbf{X}), \pi)]] \ge \inf_{\hat{\mcA}_n} \sup_{\mcP\in \mfP} \E_{\pi\sim\mcP}[\E_{\mathbf{X}\sim \pi^{\otimes n}}[\rho(\mathrm{dist}(\hat{\mcA}_n), \pi)]].
    \end{align*}
\end{theorem}
\begin{proof}
    We define the sampling algorithm $\mcA(\hpi_n)$ for every $\hpi_n$, that performs rejection sampling on the estimated density function. While rejection sampling may be computationally inefficient for the distribution, we restrict our attention to sample complexity, such that $\mcA(\hpi_n)$ samples from $\hpi_n$. The inequality follows since there can be a more sample efficient sampling algorithm that doesn't first require density estimation.
\end{proof}

Our next theorem below shows the existence of an estimator that can learn a distribution with a Bayesian risk under the TV distance of $O(\sqrt{k\log n}/\sqrt{n})$ for a $k$-sparse prior, matching our lower bounds with a poly-log term. Since density estimation is a special case of this result, Theorem \ref{thm:learn_to_sampl} allows us to conclude the existence of a sampling algorithm that achieves the same upper bound in TV distance. Similar to our lower bounds, our upper bounds are asymptotic in both $n$ and $k$. This does not make our bounds vacuous for small $k$ however, since by definition a $k$-sparse prior is always $k'$-sparse for $k'>k$. 
Our upper bounds further assume that the likelihood ratio between any two distributions is bounded by a constant $b_{lr}$. Note that the asymptotic bound does not depend on the value of $b_{lr}$, but only on its finiteness. We thus make two separate asymptotic statements for clarity. 
\begin{enumerate}
    \item Under a bounded likelihood ratio for all distributions in the support of a $k$-sparse prior, our upper and lower bounds match up to logarithmic terms when considering asymptotics in $n$ and $k$.
    \item Under a bounded likelihood ratio for all distributions in the support of a $k$-sparse prior, our upper and lower bounds match up to logarithmic terms when considering asymptotics in $n$ while keeping $k$ constant for any $k$. Here the constants hidden by the asymptotics  may not depend similarly on $k$ for the upper and lower bounds. 
\end{enumerate}

We prove the result in Appendix \ref{sec:ub_proof},and provide a proof sketch below. We show a relaxation of our bounded likelihood assumption in Corollary \ref{cor:ub_no_lr}, proceeded by a simple example that satisfies the conditions of the Corollary.

\begin{theorem}\label{thm:upper_bound}
    Consider a $k$-sparse prior in the total variation distance $\rho_{TV}$ over a set of distributions $\D$ such that the likelihood ratio between any two distributions is bounded. That is, $\|\pi_1/\pi_2\|_\infty \le b_{lr}\; \forall \pi_1,\pi_2 \in \D $, for some constant $b_{lr}>0$. Then, there exists an estimator such that the Bayes' risk is upper bounded in the total variation distance for $n=\omega(b_{lr})$ as follows:
    \begin{align*}
        \E_{\pi_0\sim\mcP}[\E_{X_1,\ldots,X_n\sim\pi_0}[\rho_{TV}(\pi_0, \hat{\pi}_n)]] = O\left(\frac{\sqrt{k\log n}}{\sqrt{n}}\right).
    \end{align*}
\end{theorem}
\begin{proof}[Proof Sketch]
    Let $E(n,\epsilon_n,b_1)$ be the event $\mcP(\{\pi:\rho_{TV}(\pi,\pi_0)>\epsilon_n|X_1,\ldots,X_n\}) > \exp(-b_1n\epsilon_n^2)$ for some $\epsilon_n$ chosen such that $n\epsilon_n^2\rightarrow\infty$ as $n\rightarrow\infty$ and constant $b_1>0$. 
    That is, $E(.)$ represents the event that the posterior mass outside of an $\epsilon_n$ ball around the true distribution sampled from $\mcP$ is greater than $\exp(-b_1n\epsilon_n^2)$. We will first argue that whenever $E(.)$ fails, there is a significant posterior probability around the true distribution, and hence we can identify a distribution close to the true distribution easily. We will then argue that the probability of $E(.)$ goes to $0$ fast as $n$ increases. Formally, we first decompose the required bound as follows, where $E(.)^c$ is the complement of the event $E(.)$:
    \begin{align}
        \E_{\pi_0\sim\mcP}[\E_{X_1,\ldots,X_n\sim\pi_0}[\rho_{TV}(\pi_0, \hat{\pi}_n)]]
        &\le Pr(E(n,\epsilon_n,b_1))\cdot \E[\rho_{TV}(\pi_0, \hat{\pi}_n)|E(n,\epsilon_n,b_1)]\nonumber\\
        &\quad +Pr(E(n,\epsilon_n,b_1)^c)\cdot \E[\rho_{TV}(\pi_0, \hat{\pi}_n)|E(n,\epsilon_n,b_1)^c].\label{eq:exp_break_main}
    \end{align}

    We define $\hat{\pi}_n = \argmax_\pi \mcP(\{\pi':\rho(\pi',\pi)\le \epsilon_n\}|X_1,\ldots, X_n)$ as the center of the $\epsilon_n$ ball with highest posterior probability. Under $E(n,\epsilon_n,b_1)^c$, there is a significant posterior mass concentrated in an $\epsilon_n$-ball around the true distribution $\pi_0$, as well as the estimated distribution $\hat{\pi}_n$. If these two distributions are more than $2\epsilon_n$ apart, and since $1-\exp(-b_1n\epsilon_n^2)$ approaches $1$ as $n\rightarrow\infty$, there is an asymptotically increasing total probability in the disjoint posterior balls around the two distributions which is eventually $>1$, which is not possible. Thus, under $E(n,\epsilon_n,b_1)^c$, $\rho_{TV}(\pi_0,\hat{\pi}_n)\le 2\epsilon_n$ asymptotically. The second term of Equation \ref{eq:exp_break_main} is thus simplified as follows:
    \begin{align*}
        Pr(E(n,\epsilon_n,b_1)^c)\cdot \E[\rho(\pi_0, \hat{\pi}_n)|E(n,\epsilon_n,b_1)^c] &\le \E[\rho(\pi_0, \hat{\pi}_n)|E(n,\epsilon_n,b_1)^c] = O(\epsilon_n).
    \end{align*}

    It remains to bound the first term of Equation \ref{eq:exp_break_main}. To do so, we first note that TV distance is bounded by 1 so that $\rho_{TV}(.,.)\le 1$. Further note from Markov's inequality that for any non-negative random variable $V$ and threshold $t>0$, $Pr(V>t) \le \E[V]/t$. Thus,
    \begin{align*}
        Pr(E(n,\epsilon_n,b_1))\cdot \E[\rho(\pi_0, \hat{\pi}_n)|E(n,&\epsilon_n,b_1)] \le Pr_{\pi_0\sim\mcP,X_1,\ldots,X_n\sim \pi_0}(E(n,\epsilon_n,b_1))\\
        &= \E_{\pi_0\sim \mcP}[Pr_{X_1,\ldots, X_n \sim \pi_0}(E(n,\epsilon_n,b_1)|\pi_0)]\\
        &\le \E_{\pi_0\sim \mcP}\left[\frac{\E[\mcP(\{\pi:\rho(\pi,\pi_0)>\epsilon_n|X_1,\ldots,X_n\})|\pi_0]}{\exp(-b_1n\epsilon_n^2)}\right].
    \end{align*}

    \noindent We use nonparametric Bayesian convergence rates developed in \citet{nonparam_bayes_vdV2000} to argue that for some $b_2,b_3>b_1$, $Pr_{\pi_0\sim \mcP}(\E[\mcP(\{\pi:\rho(\pi,\pi_0)>\epsilon_n|X_1,\ldots,X_n\})|\pi_0] > 5\exp(-b_2n\epsilon_n^2)) \le \exp(-b_3n\epsilon_n^2)$ such that for most $\pi_0$, the expected posterior probability in an $\epsilon_n$ ball around $\pi_0$ is significant.
    Since the posterior probability of any region is at most 1, this allows us to decompose the above expression:
    \begin{align*}
        \E_{\pi_0\sim \mcP}&\left[\frac{\E[\mcP(\{\pi:\rho(\pi,\pi_0)>\epsilon_n|X_1,\ldots,X_n\})|\pi_0]}{\exp(-b_1n\epsilon_n^2)}\right]\\
        &\le \exp(b_1n\epsilon_n^2)(Pr_{\pi_0\sim \mcP}(\E[\mcP(\{\pi:\rho(\pi,\pi_0)>\epsilon_n|X_1,\ldots,X_n\})|\pi_0] > 5\exp(-b_2n\epsilon_n^2)) \cdot 1\\
        &\; + Pr_{\pi_0\sim \mcP}(\E[\mcP(\{\pi:\rho(\pi,\pi_0)>\epsilon_n|X_1,\ldots,X_n\})|\pi_0] \le 5\exp(-b_2n\epsilon_n^2)) \cdot 5\exp(-b_2n\epsilon_n^2))\\
        &\le \exp(b_1n\epsilon_n^2)(\exp(-b_3n\epsilon_n^2) + 5\exp(-b_2n\epsilon_n^2)).
    \end{align*}

    The above expression is $O(\epsilon_n)$ if $\exp(-b'n\epsilon_n^2)=O(\epsilon_n)$ for any $b'>0$, which is satisfied for our choice of $\epsilon_n = \sqrt{k\log n/n}$. We defer the reader to Appendix \ref{sec:ub_proof} for a detailed proof.
\end{proof}

We now present a Corollary of the above result that allows a relaxation of the bounded likelihood assumption. In particular, since increasing $b_{lr}$ only affects the start of the asymptotic regime and not the bound itself, we allow $b_{lr}$ to increase sub-linearly in the number of samples. As long as we can ensure that such a bound on the likelihood ratio on the observed samples holds with high enough probability, we observe a similar relationship with much relaxed assumptions.

\begin{corollary}\label{cor:ub_no_lr}
    Consider a $k$-sparse prior in the total variation distance $\rho_{TV}$ over a set of distributions $\D$ supported over $\mcX$. Assume that there exists a set $\mcX_n$ such that the likelihood ratio between any two distributions is bounded inside $\mcX_n$. That is, $\|\pi_1(x)/\pi_2(x)\|_\infty \le b_{lr}(n)\; \forall \pi_1,\pi_2 \in \D, x\in\mcX_n $, for some function $b_{lr}(.):\N\rightarrow \R^{+}$ and $b_{lr}(n) = o(n)$. Further assume that for any $\pi\in\D$, $Pr_{X\sim\pi}(X\notin \mcX_n) = O\left(\frac{\sqrt{k\log n}}{n^{3/2}}\right)$.
    Then, there exists an estimator such that the Bayes' risk is upper bounded in the total variation distance as follows:
    \begin{align*}
        \E_{\pi_0\sim\mcP}[\E_{X_1,\ldots,X_n\sim\pi_0}[\rho_{TV}(\pi_0, \hat{\pi}_n)]] = O\left(\frac{\sqrt{k\log n}}{\sqrt{n}}\right).
    \end{align*}
\end{corollary}
\begin{proof}
    Denote the event that $X_1,\ldots, X_n\in \mcX_n$ by $E_n$. Then $Pr(E_n^c) \le n Pr(X_i\notin \mcX_n) = O\left(\frac{\sqrt{k\log n}}{\sqrt{n}}\right)$. Thus,
    \begin{align*}
        &\E_{\pi_0\sim\mcP}[\E_{X_1,\ldots,X_n\sim\pi_0}[\rho_{TV}(\pi_0, \hat{\pi}_n)]]\\
        &= \E_{\pi_0\sim\mcP}[\E_{X_1,\ldots,X_n\sim\pi_0}[\rho_{TV}(\pi_0, \hat{\pi}_n)|E_n]Pr(E_n) + \E_{X_1,\ldots,X_n\sim\pi_0}[\rho_{TV}(\pi_0, \hat{\pi}_n)|E_n^c]Pr(E_n^c)]\\
        &\le \E_{\pi_0\sim\mcP}[\E_{X_1,\ldots,X_n\sim\pi_0}[\rho_{TV}(\pi_0, \hat{\pi}_n)|E_n]+ Pr(E_n^c)] = O\left(\frac{\sqrt{k\log n}}{\sqrt{n}}\right),
    \end{align*}
    by Theorem \ref{thm:upper_bound}. Note that the asymptotic rate of Theorem \ref{thm:upper_bound} requires $n=\omega(b_{lr})$, so that the same asymptotic rate is achieved even when $b_{lr}$ is not a constant as long as $b_{lr}(n) = o(n)$.
\end{proof}

\begin{example}
    We show that the Gaussian location family with $k$ unknown elements of the mean satisfies the conditions of Corollary \ref{cor:ub_no_lr}. Consider the family $\D_k =\{\mcN((\theta^1,\ldots, \theta^k, 0,\ldots, 0),I_d): \theta\in\R^k,\|\theta\|\le 1\}$,
    where only the first $k$ elements of the mean are unknown. We will denote the vector $(\theta^1,\ldots, \theta^k, 0,\ldots, 0)$ by $\theta'$ for brevity here. Similar to our discussion in Section \ref{sec:sparse_dim} the W1 distance between elements of this family is given as $\rho_{W1}(\mathcal{N}(\theta_1',I_d), \mathcal{N}(\theta_2',I_d)) = \|\theta_1'-\theta_2'\| = \|\theta_1-\theta_2\|$, which allows us to conclude that any prior over the family is $k$-sparse in the W1 distance. Now, note that the TV distance is given as $\rho_{TV}(\mathcal{N}(\theta_1',I_d), \mathcal{N}(\theta_2',I_d)) = 2\Phi(\|\theta_1'-\theta_2'\|/2)-1$, where $\Phi$ is the cdf of the standard normal Gaussian. Note that the ratio $\rho_{TV}(\mathcal{N}(\theta_1',I_d), \mathcal{N}(\theta_2',I_d))/\rho_{W1}(\mathcal{N}(\theta_1',I_d), \mathcal{N}(\theta_2',I_d)) = \frac{2\Phi(r/2)-1}{r}$ for $r=\|\theta_1'-\theta_2'\|$ is decreasing in $r$. Thus for $r\in[0,2]$ given by $\|\theta\|\le 1$, we get that:
    \begin{align*}
        \rho_{TV}(\mathcal{N}(\theta_1',I_d), \mathcal{N}(\theta_2',I_d))/\rho_{W1}(\mathcal{N}(\theta_1',I_d), \mathcal{N}(\theta_2',I_d)) \in \left[\frac{2\Phi(1)-1}{2}, 1\right].
    \end{align*}
    Thus, the Gaussian family equipped with the TV distance is $k$-Ahlfors regular as defined in Corollary \ref{cor:alfhors} using the Lebesgue measure over the parameter space, which allows us to conclude that any prior over the set is $k$-sparse in the TV distance as well.

    It remains to show the existence of a support set $\mcX_n$ over which the likelihood ratios are bounded by $o(n)$. The ratio of pdfs of $\mathcal{N}(\theta_1',I_d)$ and $\mathcal{N}(\theta_2',I_d)$ at a point $x\in\R^d$ is bounded by the following:
    \begin{align*}
        \exp\left(\frac{|\|x-\theta_1'\|^2 - \|x-\theta_2'\|^2|}{2} \right) &= \exp\left(\frac{|\|x-\theta_1'\| - \|x-\theta_2'\||(\|x-\theta_1'\| + \|x-\theta_2'\|)}{2} \right)\\
        &\le \exp\left(\frac{\|\theta_1'-\theta_2'\| (2\|x\|+\|\theta_1'\| + \|\theta_2'\|)}{2} \right)\\
        &\le \exp(2\|x\|+2),
    \end{align*}
    since $\|\theta_1'\|,\|\theta_2'\|\le 1$. To obtain $b_{lr} = o(n)$, it suffices to select $\|x\|< 1/2 \log n$. Consider $\mcX_n = \{x\in\R^d:\|x\|< 1/2 \log n\} $. We recall the following Gaussian concentration inequality for an $L$-Lipschitz function $f$ and $t>0$:
    \begin{align*}
        Pr_{Z\sim \mcN(0,I_d)}(f(Z)\ge \E[f]+t) \le \exp(-t^2/(2L^2)).
    \end{align*}
    Using the triangle inequality, we know that $\|.\|$ is $1$-Lipschitz, which allows us to conclude the following for any $t\ge 0$:
    \begin{align*}
        Pr_{Z\sim \mcN(0,I_d)}(\|Z\|\ge \sqrt{d}+t) \le \exp(-t^2/2).
    \end{align*}
    Thus for a large enough $n=\omega(\exp(\sqrt{d}))$, and any $\theta$ with $\|\theta\|\le 1$, we obtain the bound:
    \begin{align*}
        Pr_{X\sim \mathcal{N}(\theta,I_d)}(X\notin\mcX_n) &= Pr_{X\sim \mathcal{N}(\theta,I_d)}(\|X\|\ge 1/2\log n)\\
        &\le Pr_{Z\sim \mcN(0,I_d)}(\|Z\|\ge 1/2\log n - \|\theta\|)\\
         &\le Pr_{Z\sim \mcN(0,I_d)}(\|Z\|\ge 1/2\log n - 1)\\
        &\le \exp\left(-\frac{(1/2\log n - 1 - \sqrt{d})^2 }{2}\right).
    \end{align*}
    Now $n=\omega(\exp(\sqrt{d}))$ implies that $Pr_{X\sim \mathcal{N}(\theta,I_d)}(X\notin\mcX_n) = O(n^{-\log n /8}) = O(\sqrt{k\log n}/n^{3/2})$, which establishes the conditions required for Corollary \ref{cor:ub_no_lr}, implying an upper bound of $O(\sqrt{k\log n}/\sqrt{n})$. We note that this derivation hides constants that depend exponentially on the dimension $d$. The upper bound $O(\sqrt{k/n})$, that has a constant dependence on $d$ is trivially known for a uniform distribution over the means \citep{all_of_stats}.
\end{example}

If we assume that every distribution in $\D$ is supported on a set with bounded diameter $\diam(\mcX)$, we obtain that for any two distributions $\pi_1,\pi_2\in\D$, $\rho_{W1}(\pi_1,\pi_2) \le \diam(\mcX)\rho_{TV}(\pi_1,\pi_2)$, which allows our bounds above to extend to the W1 distance as well. This is similar to how \citet{oko_23_diff} extend their TV bounds to W1 distance. Note that the bounds on W1 derived in this way still depend on the $k$-sparsity of the prior to hold in terms of the TV distance. 

\section{Discussion and Future Work}
We study the problem of learning a distribution given few samples from the target distribution. We note that minimax optimality results common in literature suffer from the curse of dimensionality and are too pessimistic to match empirical success. Consequently, we propose studying Bayes' optimality of learning distributions and introduce the sparse dimension as a measure of sparsity of the prior. We provide several examples of sparse distributions involving restricted support, and leave it as an open question to study sparse distributions without restricting the support. We show that the Bayes' risk of learning to sample suffers the same lower bounds that the Bayes' risk of learning distributions does, and show bounds that overcome the curse of dimensionality. We then show that the Bayes' risk of learning to sample also satisfy the same upper bounds that the Bayes' risk of learning distributions does, and show the existence of distribution estimators that match our lower bounds up to logarithmic terms. We note that while the sparse dimension of a prior might increase with the dimensionality of the support of the distributions, we overcome the curse of dimensionality in terms of the dependence of our bounds on the number of samples $n$, which goes from the worst case bounds of $\Theta(n^{-c/(d+c')})$ (for constants $c,c'$ that depend on smoothness assumptions) to $\tilde{\Theta}(\sqrt{k/n})$, where $\tilde{\Theta}$ suppresses the dependence on logarithmic terms. 
We restrict our attention to statistical efficiency in this work, and we leave a discussion on computational efficiency to future work. 
Investigating meta-learning methods to learn a feasible prior on learning distributions is an interesting direction for future work. Along similar lines, it would be interesting to explore learning algorithms that learn using a given prior, and ways to incorporate priors into gradient descent-based algorithms.

% References
\bibliographystyle{plainnat}  % Choose a bibliography style (e.g., plainnat, abbrvnat)
\bibliography{refs}     % Add your .bib file here

\begin{appendix}
\section{Information Theoretic Lower Bounds}
We first cover some results on information theory and minimax bounds in this section that will be useful to show our bounds. 

\subsection{Background on Information Theory}
For a parametrised family of distributions defined over the set $\mcX$ denoted by $\{\pi_\theta\}$ for $\theta\in \Dhp$ and a prior $\mcP$ on $\theta$, we denote the mutual information \citep{MI_csiszar72} between $\theta\sim \mcP$ and a sample $X\sim \pi_\theta$ as $I(\mcP;\Dhp)$ defined as follows:
\begin{equation*}
    I(\mcP;\Dhp) = \inf_Q \int_\Dhp D_{KL}(\pi_\theta\|Q)\mcP(d\theta),
\end{equation*}
where the infimum is over all possible distribution measures $Q$ over $\mcX$. The mutual information is a measure of the information we get from a sample $X\sim \pi_\theta$ about $\pi_\theta$, given a prior over the possible values of $\theta$.
It is widely used in information theory to provide minimax estimation lower bounds using the Fano's inequality stated below.
\begin{theorem}[Fano's inequality \citep{high_dim_stats_wainwright_19}]\label{thm:fano_ineq}
    Consider an M-ary testing problem where we want to estimate an index $\theta\in \Dhp=[M]$ corresponding to the family of distributions $\{\pi_\theta\}$ from a sample $X\sim\pi_\theta$. Further assume a prior $\mcP$ to be uniform over the set $\Dhp$, then the probability of error for any estimator $\hat{\theta}$ is lower bounded as follows:
    \begin{align*}
        Pr(\hat{\theta}(X) \ne \theta) \ge 1 - \frac{I(\mcP;\Dhp) + \log 2}{\log M}.
    \end{align*}
\end{theorem}

While we will use Fano's inequality for our proof in Appendix \ref{sec:proofs_of_lbs}, we give other relevant results that can be useful for proving lower bounds in our setting.
In the technical note \citet{cont_fano_duchi_13} extend the Fano's inequality to the continuous case where we are interested in estimating a continuous parameter $\theta$ upto an accuracy of $\epsilon$. That is, they consider the Bayes' risk of the 0-1 loss function $\id[|\theta - \hat{\theta}| \ge \epsilon]$. This is given as the continuum Fano's inequality below.
\begin{theorem}[Continuum Fano's inequality \citep{cont_fano_duchi_13}]\label{thm:cont_fano}
    Consider an estimation problem where we want to estimate $\theta\in\Dhp$ corresponding to the family of distributions $\{\pi_\theta\}$ from a sample $X\sim\pi_\theta$. Assume a 0-1 loss function given by $\id[|\theta - \hat{\theta}| \ge \epsilon]$. Further assume a prior $\mcP$ to be uniform over the set $\Dhp$, then the Bayes' risk of estimation is lower bounded as follows:
    \begin{align*}
        R_{Bayes}(\hat{\theta};\mcP) = \E_{\theta\sim\mcP}[\E_{X\sim\pi_\theta}[\id[|\theta - \hat{\theta}| \ge \epsilon]]] \ge 1 + \frac{I(\mcP;\Dhp) + \log 2}{\log(\sup_{\theta'\in\Dhp} \mcP\{\theta\in\Dhp:\|\theta-\theta'\|\le \epsilon\})}.
    \end{align*}
\end{theorem}
Note that the denominator of the second term above is a log-probability and is hence negative. This generalizes the denominator of the Fano's bound (Theorem \ref{thm:fano_ineq}) which by a change of sign of the second term is $\log(1/M)$, corresponding to a uniform prior over the $M$ hypotheses.

In the current paper we discuss bounds over non-uniform priors. The following result from \citet{bayes_risk_chen_16} which they call the Generalized Fano's inequality is useful in such cases.
\begin{theorem}[Generalized Fano's inequality \citep{bayes_risk_chen_16}]\label{thm:fano}
    Consider an estimation problem where we want to estimate $\theta\in\Dhp$ corresponding to the family of distributions $\{\pi_\theta\}$ from a sample $X\sim\pi_\theta$. Assume a 0-1 loss function given by $L(\theta,\hat{\theta})$. Further assume a prior $\mcP$ over the set $\Dhp$, then the Bayes' risk of estimation is lower bounded as follows:
    \begin{align*}
        R_{Bayes}(\hat{\theta};\mcP) = \E_{\theta\sim\mcP}[\E_{X\sim\pi_\theta}[L(\theta, \hat{\theta})]] \ge 1 + \frac{I(\mcP;\Dhp) + \log 2}{\log(\sup_{\theta'\in\Dhp} \mcP\{\theta\in\Dhp:L(\theta,\theta') = 0\})}.
    \end{align*}
\end{theorem}
It is easy to see how the statement of the Generalized Fano's inequality (Theorem \ref{thm:fano}) generalizes that of the Continuum Fano's inequality (Theorem \ref{thm:cont_fano}) by allowing for a more general prior and loss function (still binary) keeping other terms the same.

\subsection{Proofs of Main Lower Bounds}\label{sec:proofs_of_lbs}
In this Section, we give detailed proofs of our main lower bounds. We first show a dimension-independent $\Omega(\sqrt{k/n})$  lower bound on the Bayes risk for a $k$-sparse prior for the total variation distance. We will require the following result in our analysis.

\begin{lemma}[Varshamov-Gilbert bound \citep{tsybakov2009nonparam}]\label{lem:varshamov-gilbert}
    Let $m\ge 8$. There exists a sequence of binary vectors $b_0,\ldots,b_M \in \{0,1\}^m $ such that $b_0=(0,\ldots, 0)$, $\|b_i-b_j\|_1\ge m/8$ for $0\le i<j\le M$ and $M\ge 2^{m/8}$.
\end{lemma}

We are now ready to prove our result.

\begin{theorem}[Restated Theorem  \ref{thm:tv_lb}]\label{thm:tv_lb_app}
    Let $\D$ be the set of all distributions over $\mcX\subseteq\R^d$, and let $\mfP$ be shorthand for $\mfP(k,\rho_{TV},\D)$,the set of all $k$-sparse priors over $\D$ using the total variation distance. If $|\mcX|\ge k$, the Bayes' risk for any distribution estimator $\hat{\pi}_n$ is lower bounded independent of $d$ as below:
    \begin{align*}
        \inf_{\hat{\pi}}\sup_{\mcP\in \mfP} R_{Bayes}(\hat{\pi};\mcP) = \inf_{\hat{\pi}}\sup_{\mcP\in \mfP} \E_{\pi\sim\mcP}\left[\E_{\mathbf{X}\sim \pi^{\otimes n}}[\rho_{TV}(\hat{\pi}(\mathbf{X}), \pi)]\right] = \Omega(\sqrt{k/n}).
    \end{align*}
\end{theorem}
\begin{proof}
    Let $k\ge2$ and  pick $k' = \lfloor k \rfloor$ or $k' = \lfloor k \rfloor - 1$ to be an even number. Let $z_1,\ldots, z_{k'} \in \mcX$ be unequal elements in $\mcX$. Consider the set of binary vectors of length $k'/2$, $S_1 = \{-1,1\}^{k'/2} $, and for some $\delta\in(0,1/2)$, define the set of discrete probabilities $\{\theta\in[0,1]^{k'}: \theta_i = \frac{1+b_i\delta}{k'}, \theta_{k'+1-i} = \frac{1-b_i\delta}{k'}, \forall i\le k'/2, b\in S_1 \} $. Note by construction that this set contains $2^{k'/2}$ probability vectors, each of which sum to $1$ and differ in at least $2$ coordinates. In fact, we can use Lemma \ref{lem:varshamov-gilbert} for large enough  $k'$ to show that there exists a set $S_2\subseteq S_1$ of size $|S_2|\ge 2^{k'/16}$ such that any two binary vectors inside $S_2$ differ in at least $k'/16$ coordinates. Thus, define  $\Dhp=\{\theta\in[0,1]^{k'}: \theta_i = \frac{1+b_i\delta}{k'}, \theta_{k'+1-i} = \frac{1-b_i\delta}{k'}, \forall i\le k'/2, b\in S_2 \} $, as the set of $\ge2^{k'/16}$ probability vectors, each of which sum to $1$ and differ in at least $k'/8$ coordinates. Using Equation \ref{eq:disc_dist_tv}, we obtain that for $\theta,\theta'\in\Dhp$, $\rho_{TV}(\pi_\theta, \pi_{\theta'})\ge (1/2)\cdot (2\delta/k')\cdot(k'/8) = \delta/8 $. 
    
    Let $\D$ be the set of discrete distributions with probability vectors that belong to $\Dhp$, and let $\mcP$ to be the uniform prior over $\Dhp$, which is $k$-sparse by Corollary \ref{cor:d_sparse}. We simplify the desired Bayes' risk bound as follows:
    \begin{align*}
        \E_{\pi\sim\mcP}\left[\E_{\mathbf{X}\sim \pi^{\otimes n}}[\rho_{TV}(\hat{\pi}(\mathbf{X}), \pi)]\right] &= \sum_\theta \mcP(\theta)\E_{X_1,\ldots,X_n\sim \pi_\theta}[\|\hat{\theta}-\theta\|_1/2] \ge \sum_\theta \mcP(\theta) Pr(\hat{\theta}\neq \theta)\cdot \delta/8.
    \end{align*}

    Using Fano's inequality (Theorem \ref{thm:fano_ineq}),
    \begin{align*}
        Pr(\hat{\theta}\neq \theta) \ge 1-\frac{I(\mcP;\Dhp) + \log 2}{\log(2^{k'/16})}.
    \end{align*}

    In order to bound the mutual information, we first bound the KL divergence between any two elements of $\D$:
    \begin{align*}
        D_{KL}(\pi_\theta\| \pi_{\theta'}) &= \sum\theta_i\log(\theta_i/\theta'_i)= \sum_{i,\theta_i\neq \theta'_i}\theta_i\log(\theta_i/\theta'_i)  = \sum_{i, \theta_i = \frac{1+\delta}{k'}}\theta_i\log(\theta_i/\theta'_i) + \sum_{i, \theta_i = \frac{1-\delta}{k'}}\theta_i\log(\theta_i/\theta'_i)\\
        &= \sum_{i, \theta_i = \frac{1+\delta}{k'}}\theta_i\log((1+\delta)/(1-\delta)) + \sum_{i, \theta_i = \frac{1-\delta}{k'}}\theta_i\log((1-\delta)/(1+\delta)).
    \end{align*}
    Now, the number of times $\theta_i>\theta'_i$ is equal to the number of times $\theta_i<\theta'_i$ since the sum over all $\theta_i$ and $\theta_i'$ is constant. Further, for $\delta\le 1/2$, $\log((1+\delta)/(1-\delta)) \le 3\delta$.
    Thus, 
    \begin{align*}
        D_{KL}(\pi_\theta\| \pi_{\theta'})&= \sum_{i, \theta_i = \frac{1+\delta}{k'}}(\delta/k')\log((1+\delta)/(1-\delta)) + \sum_{i, \theta_i = \frac{1-\delta}{k'}}(-\delta/k')\log((1-\delta)/(1+\delta))\\
        &\le \frac{1}{k'}(\sum_{i, \theta_i = \frac{1+\delta}{k'}}3\delta^2 + \sum_{i, \theta_i = \frac{1-\delta}{k'}}3\delta^2)\le 3\delta^2.
    \end{align*}

    This allows the following mutual information bound, for some $\theta_0\in\Dhp$:
    \begin{align*}
        I(\mcP;\Dhp) &= \inf_{Q \text{ over } \mcX^{\otimes n} }\sum D_{KL}(\pi_\theta^{\otimes n}\| Q )\mcP(\theta) \le \sum D_{KL}(\pi_\theta^{\otimes n}\| \pi_{\theta_0}^{\otimes n} )\mcP(\theta)\\
        &\le n\sum D_{KL}(\pi_\theta\| \pi_{\theta_0})\mcP(\theta) \le 3n\delta^2.
    \end{align*}

    Substituting this into our desired bound gives:
    \begin{align*}
        \E_{\pi\sim\mcP}\left[\E_{\mathbf{X}\sim \pi^{\otimes n}}[\rho_{TV}(\hat{\pi}(\mathbf{X}), \pi)]\right] &\ge (\delta/8)\sum_\theta \mcP(\theta) Pr(\hat{\theta}\neq \theta) \ge (\delta/8)\sum_\theta \mcP(\theta) \left(1-\frac{3n\delta^2+\log 2}{(k'/16)\log 2}\right)\\
        &= (\delta/8)\left(1-\frac{3n\delta^2+\log 2}{(k'/16)\log 2}\right).
    \end{align*}

    Now, since the choice of $\delta$ is arbitrary, for large enough $k'$ and $n = \Omega(k)$, we pick $\delta = \sqrt{\frac{((k'/32) - 1)\log 2}{3n}} $ to obtain our desired bound of $\Omega(\sqrt{k'/n}) = \Omega(\sqrt{k/n})$.
\end{proof}

\section{Nonparametric Bayesian Upper Bounds}
We discuss Bayesian distribution estimation in this paper using nonparametric priors. We derive our upper bounds using the theory of posterior convergence rates from \citet{nonparam_bayes_vdV2000}. We state important results from nonparametric Bayesian convergence rates here, but first give a brief background on covering numbers.

\subsection{Background on Covering Numbers}\label{sec:covering}
Covering numbers and packing numbers are important tools used in learning theoretic upper bounds \citep{ML_theory_Shwartz_14, nonparam_bayes_vdV2000}. We recall the definition of covering number from Section \ref{sec:sparse_dim}.

\begin{definition}[Covering number]
    For a set $\D$ with a distance (pseudo-)metric $\rho$, consider a set of points $\A$ such that for any element $\mu\in\D$, $\exists \nu \in \A$ such that $\rho(\mu,\nu)\le r$. Then $\A$ is called an $r$-cover of $\D$. We define $N(r,\D)$, the $r$-covering number of $\D$, as the cardinality of the smallest $r$-cover of $\D$. Thus, if the $r$-cover $\A$ is the smallest possible, $|\A| = N(r,\D)$.
\end{definition}

An analogous notion to the covering number is the packing number.

\begin{definition}[Packing number]
    For a set $\D$ with a distance (pseudo-)metric $\rho$, consider a set of points $\A\subseteq \D$ such that $\nexists\; \nu_1\neq\nu_2\in\A$ for which $\rho(\nu_1,\nu_2)\le r$. Then $\A$ is called an $r$-packing of $\D$. We define $D(r,\D)$, the $r$-packing number of $\D$, as the cardinality of the biggest $r$-packing of $\D$. Thus, if the set $\A$ is the largest possible, $|\A| = D(r,\D)$.
\end{definition}

We state the following Lemma, without proof which is commonly used to relate covering numbers with packing numbers.

\begin{lemma}[\citep{nonparam_bayes_vdV2000}]\label{lem:cover_packing}
    For a set $\D$ with a distance (pseudo-)metric $\rho$ and for any $r>0$,
    \begin{align*}
        N(r,\D)\le D(r,\D)\le N(r/2,\D).
    \end{align*}
\end{lemma}

We use the above Lemma to show the Proposition below which is useful for our analysis in the paper.

\begin{proposition}\label{prop:cover_d_dim}
    Let $\D \subseteq \R^d$ equipped with an $l_a$ norm for $a\in[1,\infty]$ as the distance metric. For any point $\nu\in\D$ and $0<r<R$, $N(r,B(R,\nu)\cap\D)\le (3R/r)^d$.
\end{proposition}
\begin{proof}
    We first note that in a $d$-dimensional space with any $l_a$ norm, the volume of any $r$-ball,  $V(r)\propto r^d$. Consider the largest $r$ packing of $B(R,\nu)\cap\D$, such that every point in the packing is at a distance at least $r$ from each other. Thus, balls of radius $r/2$ created around every element of the packing must be non-overlapping. If we have $N$ such balls in the largest packing, then the total volume of the balls, $NV(r/2)$, is less than the volume of the $R+r/2$ ball, $V(R+r/2)$, because the balls don't overlap. Thus, $D(r,B(R,\nu)\cap\D)=N\le (R+r/2)^d/(r/2)^d = (2R/r + 1)^d \le (3R/r)^d$. Now using Lemma \ref{lem:cover_packing} completes the proof.
\end{proof}

We use a similar idea to bound the covering number of a $k$-Ahlfors regular set. We first define a $k$-Ahlfors regular set formally.

\begin{definition}[$k$-Ahlfors regular sets]\label{def:ahlfors}
    A set $\D$ equipped with the distance (pseudo-)metric $\rho$ is $k$-Ahlfors regular if there exists a Borel measure $\mathcal{M}$ with support $\D$ and a constant $C\ge 1$ such that for any $\nu\in\D$ and radius $0<r \le \diam(\D)$, $C^{-1}r^k \le \mathcal{M}(B(r,\nu)\cap \D) \le Cr^k $.
\end{definition}

The following Proposition bounds the covering number of a set given that it is Ahlfors regular.

\begin{proposition}\label{prop:ahlfors_cover}
    Consider a $k$-Ahlfors regular set $\D$ equipped with the distance (pseudo-)metric $\rho$ such that the Borel measure $\mathcal{M}$ satisfies the conditions of Definition \ref{def:ahlfors}. For any point $\nu\in\D$ and $0<r<R\le \diam(\D)$, $N(r,B(R,\nu)\cap \D)\le C^2(3R/r)^k$.
\end{proposition}
\begin{proof}
    Consider the largest $r$ packing of $B(R,\nu)\cap \D$, such that every point in the packing is at a distance at least $r$ from each other. Thus, balls of radius $r/2$ created around every element of the packing (each of which is in $\D$) must be non-overlapping. If we have $N$ such balls in the largest packing, then the total $\mathcal{M}$-measure of the balls is always less than the total $\mathcal{M}$-measure  of the $R+r/2$ ball centered at $\nu$, because the balls don't overlap. Thus, $NC^{-1}(r/2)^k \le C(R+r/2)^k \implies N\le C^2(R+r/2)^k/(r/2)^k \le C^2(3R/r)^k$. Since $D(r,B(R,\nu)\cap \D)=N$, using Lemma \ref{lem:cover_packing} completes the proof.
\end{proof}

The following result will be useful for showing Ahlfors regularity of sets.

\begin{lemma}\label{lem:ahlfors_lipschitz}
    Consider (pseudo-)metrics $\rho_1,\rho_2$ such that there exist constants $c_1,c_2>0$ for which for any $\mu_1,\mu_2\in\D$, $c_1\rho_2(\mu_1,\mu_2) \le \rho_1(\mu_1,\mu_2) \le c_2 \rho_2(\mu_1,\mu_2) $. If the set $\D$ is $k$-Ahlfors regular when equipped with the (pseudo-)metric $\rho_1$, it is also $k$-Ahlfors regular when equipped with the (pseudo-)metric $\rho_2$.
\end{lemma}
\begin{proof}
    Let $\mathcal{M}$ be the Borel measure guaranteed by Definition \ref{def:ahlfors}, such that there exists a constant $C\ge 1$ such that for any $\nu\in\D$ and radius $0<r \le \diam_{\rho_1}(\D)$, $C^{-1}r^k \le \mathcal{M}(B_{\rho_1}(r,\nu)\cap \D) \le Cr^k $.
    Now, for any $\nu\in\D$ and radius $0<r \le \diam_{\rho_2}(\D)$, we have that $\mathcal{M}(B_{\rho_2}(r,\nu)\cap \D) \le \mathcal{M}(B_{\rho_1}(c_2r,\nu)\cap \D) \le C\min(c_2r,\diam_{\rho_1}(\D))^k \le Cc_2^kr^k$. Now note that  $c_1\diam_{\rho_2}(\D) \le \diam_{\rho_1}(\D)$, since otherwise there exist points $\mu_1,\mu_2\in\D$ for which $c_1\rho_2(\mu_1,\mu_2) > \rho_1(\mu_1,\mu_2)$. Thus, for any $\nu\in\D$ and radius $0<r \le \diam_{\rho_2}(\D)$, we have that $\mathcal{M}(B_{\rho_2}(r,\nu)\cap \D) \ge \mathcal{M}(B_{\rho_1}(c_1r,\nu)\cap \D) \ge C^{-1}(c_1r)^k $. We establish the required Ahlfors regularity for $\rho_2$ by picking a constant $C' = \max(Cc_1^{-k}, Cc_2^k)$.
\end{proof}

The following result from \citet{nonparam_bayes_vdV2000} shows the existence of a test for a distribution $\nu_0$, given a bound on the packing number of a shrinking ball around $\nu_0$. Here we denote a test as a measurable function $\phi:\mcX^{n}\rightarrow[0,1]$.

\begin{lemma}[Theorem 7.1, \citep{nonparam_bayes_vdV2000}]\label{lem:pack_tests}
    Suppose that for a non-increasing function $D(\epsilon)$, some $\epsilon_n\ge0$ and every $\epsilon>\epsilon_n$,
    \begin{align*}
        D(\epsilon/2, \{\nu:\epsilon\le \rho(\nu,\nu_0)\le 2\epsilon\}) \le D(\epsilon),
    \end{align*}
    where $\rho$ is the distance (pseudo-)metric between distributions chosen to be either the Hellinger or the TV distance. Then for every $\epsilon>\epsilon_n$, there exist tests $\phi_n$ (depending on $\epsilon>0$), such that, for a universal constant $C_\mathrm{test}$ and every $j\in\N$,
    \begin{align*}
        \E_{\nu_0^{\otimes n}}[\phi_n] &\le D(\epsilon)\exp(-C_\mathrm{test}n\epsilon^2)\frac{1}{1-\exp(-C_\mathrm{test}n\epsilon^2)}\\
        \sup_{\rho(\nu,\nu_0)>j\epsilon} \E_{\nu^{\otimes n}}[1-\phi_n] &\le \exp(-C_\mathrm{test}n\epsilon^2j^2).
    \end{align*}
\end{lemma}

\subsection{Posterior Contraction Rates}
We begin this section by defining the posterior rate of contraction for a fixed prior. This definition is a weaker form of the treatment of \citet{nonparam_bayes_vdV2000}, where the authors allow the prior to depend on the number of samples $n$.

\begin{definition}[Posterior Rate of Contraction]
    Consider a set of distributions $\{\pi_\theta\}$ parametrized by $\theta\in\D$ with a distance (pseudo-)metric $\rho$ over $\theta$, and a prior distribution $\mcP$ over the set $\D$. Assume we obtain samples $X_1,\ldots, X_n \sim \pi_{\theta_0} $ for $\theta_0\in\D$ and denote by $\mcP(\cdot|X_1,\ldots,X_n)$ the posterior probability of an event. We say the posterior contracts at a rate $\epsilon_n\rightarrow0$ if  $\mcP(\theta:\rho(\theta,\theta_0)>M_n\epsilon_n|X_1,\ldots, X_n) \rightarrow0$ in $\pi_{\theta_0}^{\otimes n}$-probability, for every $M_n\rightarrow \infty$ as $n\rightarrow \infty$.
\end{definition}

Intuitively, the above definition defines the posterior to contract at a rate $\epsilon_n$ if the posterior probability of an $\epsilon_n$-ball around the true parameter tends to 1. Note that asymptotic rates for parametric Bayesian approaches do not depend on the priors since we expect any estimator to ``forget'' the prior in the asymptotic limit of samples. This is however not true for nonparametric Bayesian approaches as considered in this paper, where the asymptotic rates depend heavily on the choice of prior as we observe in the following result. We state the following Lemma, which is easily derivable from the proof of Theorem 2.1 in \citet{nonparam_bayes_vdV2000}, proved formally here for the reader's convenience.

\begin{lemma}\label{lem:contraction}
    Consider a prior $\mcP$ defined over the set of distributions $\D$ with the distance metric $\rho$ chosen as either the Hellinger or the TV distance. Suppose there exists sequences $\delta_n$ and $\epsilon_n$ with $\epsilon_n\rightarrow0$ and $n\epsilon_n^2\rightarrow\infty$ as $n\rightarrow\infty$, and a constant $b>0$ and sets $\D_n\subset \D$ such that for $\pi_0\in\D$,
    \begin{enumerate}
    \item \begingroup\abovedisplayskip=0pt\belowdisplayskip=0pt\begin{align*}
        \log(D(\epsilon_n,\D_n))\le n\epsilon_n^2
    \end{align*}\endgroup
    \item \begingroup\abovedisplayskip=0pt\belowdisplayskip=0pt\begin{align*}
        \mcP(\D/\D_n) \le \exp(-n\epsilon_n^2(b+4))
    \end{align*}\endgroup
    \item \begingroup\abovedisplayskip=0pt\belowdisplayskip=0pt
    \begin{align*}
        Pr_{X_1,\ldots, X_n\sim \pi_0}\left(\int \prod (\pi/\pi_0)(X_i)\dd\mcP(\pi) \ge \exp(-(2+b)n\epsilon_n^2)\right) \ge 1-\delta_n.
    \end{align*}\endgroup
\end{enumerate}
Then, there exists a universal constant $C_{\mathrm{test}}$ for which we can bound the posterior probability far from $\pi_0$ for a large enough $M>0$, asymptotically as $n\rightarrow \infty$, as follows:
\begin{align*}
    \E_{\pi_0^{\otimes n}}[\mcP(\pi:\rho(\pi,\pi_0)&>M\epsilon_n|X_1,\ldots,X_n)] \le 2\exp(-C_{\mathrm{test}}n\epsilon_n^2) + \delta_n\\
    &+ \exp((2+b)n\epsilon_n^2)(\exp(-(b+4)n\epsilon_n^2) + \exp(-C_{\mathrm{test}}nM^2\epsilon_n^2)).
\end{align*}
\end{lemma}
\begin{proof}
    The proof tracks the proof of Theorem 2.1 in \citet{nonparam_bayes_vdV2000} closely, adding some more detail for clarity. For every $\epsilon > 2\epsilon_n$, Condition 1 implies:
    \begin{align*}
        \log(D(\epsilon/2,\D_n)) \le \log(D(\epsilon_n,\D_n)) \le n\epsilon_n^2.
    \end{align*}
    Thus by Lemma \ref{lem:pack_tests} using $D(\epsilon) = \exp(n\epsilon_n^2)$ (constant in $\epsilon$), $\epsilon = M\epsilon_n$ for large enough $M\ge 2$, and $j=1$, there exists tests $\phi_n$ such that
    \begin{align}
        \E_{\pi_0^{\otimes n}}[\phi_n] &\le \exp(n\epsilon_n^2)\exp(-C_\mathrm{test}nM^2\epsilon_n^2)\frac{1}{1-\exp(-C_\mathrm{test}nM^2\epsilon_n^2)}\nonumber\\
        \sup_{\pi\in\D_n, \rho(\pi,\pi_0)>M\epsilon_n} \E_{\pi^{\otimes n}}[1-\phi_n] &\le \exp(-C_\mathrm{test}nM^2\epsilon_n^2).\label{eq_step:contraction_test}
    \end{align}

    Now if $C_\mathrm{test}M^2 - 1 > C_\mathrm{test}$, the first condition implies the following asymptotically as $n\rightarrow \infty$,
    \begin{align*}
        \E_{\pi_0^{\otimes n}}[\mcP(\pi:\rho(\pi,\pi_0)> M\epsilon_n | X_1,\ldots, X_n )\phi_n] \le \E_{\pi_0^{\otimes n}}[\phi_n] \le 2\exp(-C_\mathrm{test} n\epsilon_n^2).
    \end{align*}

    Given the above bound, we now only need to upper bound $\E_{\pi_0^{\otimes n}}[\mcP(\pi:\rho(\pi,\pi_0)> M\epsilon_n | X_1,\ldots, X_n )(1 -\phi_n)]$ to finish the proof. We expand the expression $\mcP(\pi:\rho(\pi,\pi_0)> M\epsilon_n | X_1,\ldots, X_n )$ below:
    \begin{align*}
        \mcP(\pi:\rho(\pi,\pi_0)> M\epsilon_n | X_1,\ldots, X_n ) %&= \frac{Pr(X_1,\ldots, X_n|\pi:\rho(\pi,\pi_0)> M\epsilon_n) Pr_\mcP(\pi:\rho(\pi,\pi_0)> M\epsilon_n) }{Pr(X_1,\ldots, X_n)}\\
        &= \frac{\int_{\pi:\rho(\pi,\pi_0)> M\epsilon_n} \prod_{i=1}^n \pi(X_i) \dd\mcP(\pi) }{\int \prod_{i=1}^n \pi(X_i) \dd\mcP(\pi)}\\
        &= \frac{\int_{\pi:\rho(\pi,\pi_0)> M\epsilon_n} \prod_{i=1}^n (\pi/\pi_0)(X_i) \dd\mcP(\pi) }{\int \prod_{i=1}^n (\pi/\pi_0)(X_i) \dd\mcP(\pi)}.
    \end{align*}

    Let $E_n$ be the event that $\int \prod_{i=1}^n (\pi/\pi_0)(X_i) \dd\mcP(\pi) \ge \exp(-(2+b)n\epsilon_n^2)$, which occurs with probability $\ge 1-\delta_n$ by condition 3. Thus, given $E_n$ happens we can say the following,
    \begin{align*}
        \mcP(\pi:\rho(\pi,\pi_0)> M\epsilon_n &| X_1,\ldots, X_n )\\
        &\le\exp((2+b)n\epsilon_n^2) \int_{\pi:\rho(\pi,\pi_0)> M\epsilon_n} \prod_{i=1}^n (\pi/\pi_0)(X_i) \dd\mcP(\pi)\\
        &= \exp((2+b)n\epsilon_n^2)\bigg[\int_{\pi\notin\D_n:\rho(\pi,\pi_0)> M\epsilon_n} \prod_{i=1}^n (\pi/\pi_0)(X_i) \dd\mcP(\pi)\\
        &\quad + \int_{\pi\in\D_n:\rho(\pi,\pi_0)> M\epsilon_n} \prod_{i=1}^n (\pi/\pi_0)(X_i) \dd\mcP(\pi)\bigg]\\
        &\le \exp((2+b)n\epsilon_n^2)\bigg[\int_{\pi\notin\D_n} \prod_{i=1}^n (\pi/\pi_0)(X_i) \dd\mcP(\pi)\\
        &\quad + \int_{\pi\in\D_n:\rho(\pi,\pi_0)> M\epsilon_n} \prod_{i=1}^n (\pi/\pi_0)(X_i) \dd\mcP(\pi)\bigg].
    \end{align*}
    Substituting into the desired expectation bound,
    \begin{align*}
        \E[\mcP(\pi:\rho(\pi,\pi_0)> M\epsilon_n &| X_1,\ldots, X_n )(1-\phi_n)|E_n] \\
        &\le\exp((2+b)n\epsilon_n^2) \bigg[\E_{\pi_0^{\otimes n}}[ \int_{\pi\notin\D_n} \prod_{i=1}^n (\pi/\pi_0)(X_i)(1-\phi_n) \dd\mcP(\pi)]\\
        &\quad + \E_{\pi_0^{\otimes n}}[\int_{\pi\notin\D_n:\rho(\pi,\pi_0)> M\epsilon_n} \prod_{i=1}^n (\pi/\pi_0)(X_i) (1-\phi_n)\dd\mcP(\pi)]\bigg]\\
        &\le \exp((2+b)n\epsilon_n^2) \bigg[\E_{\pi_0^{\otimes n}}[ \int_{\pi\notin\D_n} \prod_{i=1}^n (\pi/\pi_0)(X_i) \dd\mcP(\pi)]\\
        &\quad + \E_{\pi_0^{\otimes n}}[\int_{\pi\notin\D_n:\rho(\pi,\pi_0)> M\epsilon_n} \prod_{i=1}^n (\pi/\pi_0)(X_i) (1-\phi_n)\dd\mcP(\pi)]\bigg]\\
        &\le \exp((2+b)n\epsilon_n^2) \bigg[\int_{\pi\notin\D_n} 1 \dd\mcP(\pi)\\
        &\quad + \int_{\pi\notin\D_n:\rho(\pi,\pi_0)> M\epsilon_n} (1-\phi_n)\dd\mcP(\pi)\bigg].
    \end{align*}

    In the last step above we invoke Fubini's theorem to swap the integrals and then note that $\E_{\pi_0^{\otimes n}}[\prod (\pi/\pi_0)(X_i)] \le 1$. Both the terms of the obtained inequality are now easy to bound using Equation \ref{eq_step:contraction_test} for the second term, and condition 2 for the first term since $\int_{\pi\notin\D_n} 1 \dd\mcP(\pi) = \mcP(\D/\D_n) $. We finish the proof by noting that $\E[\mcP(\pi:\rho(\pi,\pi_0)> M\epsilon_n | X_1,\ldots, X_n )(1-\phi_n)|E_n^c]Pr(E_n^c) \le Pr(E_n^c) \le \delta_n$.
\end{proof}

Note that for $\delta_n\rightarrow0$ and large enough $M$, the above bound directly implies that the posterior contracts at a rate $\epsilon_n$. The following Lemma allows us to show $\delta_n\rightarrow0$ for a set of distributions that are close in Hellinger distance. First, we remind the reader that the Hellinger distance between distributions $\pi_1,\pi_2$ with pdfs written as  $\pi_1(.), \pi_2(.)$ respectively is defined as \citep{high_dim_stats_wainwright_19}:
\begin{align*}
    h^2(\pi_1,\pi_2) &= \int\left(\sqrt{\pi_1(x)} - \sqrt{\pi_2(x)}\right)^2\dd x. 
\end{align*}
The Hellinger distance is closely related to the TV distance as follows:
\begin{align}
   \frac{1}{2} h^2(\pi_1,\pi_2) \le \rho_{TV}(\pi_1,\pi_2) \le h(\pi_1,\pi_2)\sqrt{1-\frac{h^2(\pi_1,\pi_2)}{4}}.\label{eq:hell_to_tv}
\end{align}

\begin{lemma}[Lemma 8.4 in \citet{nonparam_bayes_vdV2000}]\label{lem:hellinger_ball}
    For every $\epsilon>0$ and distribution $\mcP$ on the set $\D = \{\pi:h^2(\pi,\pi_0)\|\pi_0/\pi\|_\infty \le \epsilon^2\} $, we have, for a universal constant $C_{\mathrm{hellinger}}>0$,
    \begin{align*}
        Pr_{X_1,\ldots, X_n\sim \pi_0}\left(\int \prod (\pi/\pi_0)(X_i)\dd\mcP(\pi) \le \exp(-3n\epsilon^2)  \right) \le \exp(-C_{\mathrm{hellinger}}n\epsilon^2).
    \end{align*}
\end{lemma}

% Finally, we state the Borel-Cantelli lemma which is useful in our analysis.

% \begin{lemma}[Borel-Cantelli Lemma \citep{durrett2010probability}]
%     For a sequence of events $E_n$, if $\sum Pr(E_n)<\infty$, then $Pr(E_n \text{ happens infinitely often}) = 0$.
% \end{lemma}

% The lemma states that if the sum of probabilities of individual events is finite, then only a finite number of these events occur almost surely. In other words, there exists a large enough $n_0$ almost surely, such that none of the events $E_n$ for $n>n_0$ occur. When proving our rates of convergence, we will consider events $E_n$ that the rate of convergence at $n$ samples is slower than $\epsilon_n$, and obtain that this is true only for $\epsilon_n'$ probability of the prior. We will show $\sum \epsilon_n'<\infty$ to show that almost surely under the prior, the rate of convergence is asymptotically $\epsilon_n$.

\subsection{Proof of Theorem \ref{thm:upper_bound}}\label{sec:ub_proof}

We now prove our main upper bound that assumes that the likelihood ratio between any two distributions is bounded by a constant $b_{lr}$. Note that our asymptotic bound does not depend on the value of $b_{lr}$ and only on its finiteness.

\begin{theorem}[Restated Theorem \ref{thm:upper_bound}]
    Consider a $k$-sparse prior in the total variation distance $\rho_{TV}$ over a set of distributions $\D$ such that the likelihood ratio between any two distributions is bounded. That is, $\|\pi_1/\pi_2\|_\infty \le b_{lr}\; \forall \pi_1,\pi_2 \in \D $, for some constant $b_{lr}>0$. Then, there exists an estimator such that the Bayes' risk is upper bounded in the total variation distance for $n=\omega(b_{lr})$ as follows:
    \begin{align*}
        \E_{\pi_0\sim\mcP}[\E_{X_1,\ldots,X_n\sim\pi_0}[\rho_{TV}(\pi_0, \hat{\pi}_n)]] = O\left(\frac{\sqrt{k\log n}}{\sqrt{n}}\right).
    \end{align*}
\end{theorem}
\begin{proof}
    We will shorthand $\rho_{TV}$ for $\rho$ in this proof for notational convenience. 
    Let $E(n,\epsilon_n,b_1)$ be the event $\mcP(\{\pi:\rho(\pi,\pi_0)>\epsilon_n|X_1,\ldots,X_n\}) > \exp(-b_1n\epsilon_n^2)$ for some $\epsilon_n$ chosen such that $n\epsilon_n^2\rightarrow\infty$ as $n\rightarrow\infty$ and constant $b_1>0$. 
    That is, $E(.)$ represents the event that the posterior mass outside of an $\epsilon_n$ ball around the true distribution sampled from $\mcP$ is greater than $\exp(-b_1n\epsilon_n^2)$. We will first argue that whenever $E(.)$ fails, there is a significant posterior probability around the true distribution, and hence we can identify a distribution close to the true distribution easily. We will then argue that the probability of $E(.)$ goes to $0$ fast as $n$ increases. Formally, we first decompose the required bound as follows, where $E(.)^c$ is the complement of the event $E(.)$:
    \begin{align}
        \E_{\pi_0\sim\mcP}[\E_{X_1,\ldots,X_n\sim\pi_0}[\rho_{TV}(\pi_0, \hat{\pi}_n)]]
        &\le Pr(E(n,\epsilon_n,b_1))\cdot \E[\rho(\pi_0, \hat{\pi}_n)|E(n,\epsilon_n,b_1)]\nonumber\\
        &\quad +Pr(E(n,\epsilon_n,b_1)^c)\cdot \E[\rho(\pi_0, \hat{\pi}_n)|E(n,\epsilon_n,b_1)^c].\label{eq:exp_break}
    \end{align}

    We define $\hat{\pi}_n = \argmax_\pi \mcP(\{\pi':\rho(\pi',\pi)\le \epsilon_n\}|X_1,\ldots, X_n)$ as the center of the $\epsilon_n$ ball with highest posterior probability. Consider $E(n,\epsilon_n,b_1)^c$ so that $\mcP(\{\pi:\rho(\pi,\pi_0)\le\epsilon_n|X_1,\ldots,X_n\}) >  1 - \exp(-b_1n\epsilon_n^2)$, and further assume that $\rho(\pi_0,\hat{\pi}_n)> 2\epsilon_n$ such that the $\epsilon_n$-posterior balls around $\pi_0$ and $\hat{\pi}_n$ are disjoint. By definition of $\hat{\pi}_n$, $\mcP(\{\pi:\rho(\pi,\hat{\pi}_n)\le\epsilon_n|X_1,\ldots,X_n\}) \ge \mcP(\{\pi:\rho(\pi,\pi_0)\le\epsilon_n|X_1,\ldots,X_n\}) >  1 - \exp(-b_1n\epsilon_n^2)$. Now, since $n\epsilon_n^2\rightarrow\infty$, $\exists n_1>0$ such that for every $n>n_1$, $\mcP(\{\pi:\rho(\pi,\hat{\pi}_n)\le\epsilon_n|X_1,\ldots,X_n\}) + \mcP(\{\pi:\rho(\pi,\pi_0)\le\epsilon_n|X_1,\ldots,X_n\}) > 1$, which is not possible if the balls are disjoint. In other words, for every $n>n_1$, $E(n,\epsilon_n,b_1)^c\implies \rho(\pi_0,\hat{\pi}_n)\le 2\epsilon_n$. We thus simplify the second term of Equation \ref{eq:exp_break} as follows:
    \begin{align*}
        Pr(E(n,\epsilon_n,b_1)^c)\cdot \E[\rho(\pi_0, \hat{\pi}_n)|E(n,\epsilon_n,b_1)^c] &\le \E[\rho(\pi_0, \hat{\pi}_n)|E(n,\epsilon_n,b_1)^c] = O(\epsilon_n).
    \end{align*}

    It remains to bound the first term of Equation \ref{eq:exp_break}. To do so, we first note that TV distance is bounded by 1 so that $\rho_{TV}(.,.)\le 1$. Further note from Markov's inequality that for any non-negative random variable $V$ and threshold $t>0$, $Pr(V>t) \le \E[V]/t$. Thus,
    \begin{align*}
        Pr(E(n,\epsilon_n,b_1))\cdot \E[\rho(\pi_0, \hat{\pi}_n)|E(n,&\epsilon_n,b_1)] \le Pr_{\pi_0\sim\mcP,X_1,\ldots,X_n\sim \pi_0}(E(n,\epsilon_n,b_1))\\
        &= \E_{\pi_0\sim \mcP}[Pr_{X_1,\ldots, X_n \sim \pi_0}(E(n,\epsilon_n,b_1)|\pi_0)]\\
        &\le \E_{\pi_0\sim \mcP}\left[\frac{\E[\mcP(\{\pi:\rho(\pi,\pi_0)>\epsilon_n|X_1,\ldots,X_n\})|\pi_0]}{\exp(-b_1n\epsilon_n^2)}\right].
    \end{align*}

    Looking at the result of Lemma \ref{lem:contraction}, we know we can get a bound of the sort $\E[\mcP(\{\pi:\rho(\pi,\pi_0)>\epsilon_n|X_1,\ldots,X_n\})|\pi_0] \le 5\exp(-b_2n\epsilon_n^2)$ for some $b_2>0$ for $\pi_0$ that satisfies the conditions of the Lemma. Let $Pr_{\pi_0\sim \mcP}(\E[\mcP(\{\pi:\rho(\pi,\pi_0)>\epsilon_n|X_1,\ldots,X_n\})|\pi_0] > 5\exp(-b_2n\epsilon_n^2)) \le \exp(-b_3n\epsilon_n^2)$ for some $b_3>0$ so that most $\pi_0$ satisfy the conditions of Lemma \ref{lem:contraction} at any given $n$.
    Since the posterior probability of any region is at most 1, this allows us to decompose the above expression:
    \begin{align*}
        &\E_{\pi_0\sim \mcP}\left[\frac{\E[\mcP(\{\pi:\rho(\pi,\pi_0)>\epsilon_n|X_1,\ldots,X_n\})|\pi_0]}{\exp(-b_1n\epsilon_n^2)}\right]\\
        &\le \exp(b_1n\epsilon_n^2)(Pr_{\pi_0\sim \mcP}(\E[\mcP(\{\pi:\rho(\pi,\pi_0)>\epsilon_n|X_1,\ldots,X_n\})|\pi_0] > 5\exp(-b_2n\epsilon_n^2)) \cdot 1\\
        &\; + Pr_{\pi_0\sim \mcP}(\E[\mcP(\{\pi:\rho(\pi,\pi_0)>\epsilon_n|X_1,\ldots,X_n\})|\pi_0] \le 5\exp(-b_2n\epsilon_n^2)) \cdot 5\exp(-b_2n\epsilon_n^2))\\
        &\le \exp(b_1n\epsilon_n^2)(\exp(-b_3n\epsilon_n^2) + 5\exp(-b_2n\epsilon_n^2)).
    \end{align*}

    The above expression is $O(\epsilon_n)$ if $b_2,b_3>b_1$ and $\exp(-b'n\epsilon_n^2)=O(\epsilon_n)$ for any $b'>0$. This argument thus shows the required bound if the following conditions are satisfied:
    \begin{enumerate}
        \item The contraction rate $\epsilon_n$ satisfies $n\epsilon_n^2\rightarrow\infty$ as $n\rightarrow\infty$ and $\exp(-b'n\epsilon_n^2)=O(\epsilon_n)$ for any $b'>0$.
        \item For some $b_2,b_3>0$, 
        \begin{align*}
            Pr_{\pi_0\sim \mcP}(\E[\mcP(\{\pi:\rho(\pi,\pi_0)>\epsilon_n|X_1,\ldots,X_n\})|\pi_0] > 5\exp(-b_2n\epsilon_n^2)) \le \exp(-b_3n\epsilon_n^2).
        \end{align*}
        Further there is a $b_1\in(0,\min(b_2,b_3))$.
    \end{enumerate}

    Note that the condition on $b_1$ is trivially satisfied. Further, note that there exist $n_2,k_2>0$ such that for all $n>n_2,k>k_2$ and constant $b_4>0$, $n^{-b'b_4^2k}\le b_4\frac{\sqrt{k\log n}}{\sqrt{n}}$. Thus the contraction rate $\epsilon_n = b_4\frac{\sqrt{k\log n}}{\sqrt{n}}$ for some constant $b_4>0$ as required in the Theorem statement satisfies the first condition. It remains to show that there are $b_2,b_3>0$ for which $Pr_{\pi_0\sim \mcP}(\E[\mcP(\{\pi:\rho(\pi,\pi_0)>\epsilon_n|X_1,\ldots,X_n\})|\pi_0] > 5\exp(-b_2n\epsilon_n^2)) \le \exp(-b_3n\epsilon_n^2)$. As noted before, we show this using Lemma \ref{lem:contraction} in the asymptotic limit, where we bound probability under the prior $\mcP$ that the sampled distribution doesn't satisfy the conditions 1-3 of the Lemma for big enough $n$. Conditions 1-2 in the Lemma do not depend on the prior or the sampled $\pi_0$, and only require a bound on the packing number of a significant fraction of the support space. 

    \begin{enumerate}[leftmargin=*,topsep=0pt,partopsep=1ex,parsep=1ex]\itemsep=-4pt
        \item \textbf{Proving conditions 1-2 of Lemma \ref{lem:contraction} for any constant $b>0$ and $\epsilon_n = b_4\frac{\sqrt{k\log n}}{\sqrt{n}}$:}

        \noindent We need to bound the $\epsilon_n$ packing number for some subset $\D_n$. First, note that $\diam(\D)\le 1$ since the total variation distance between any two distributions is at most 1. Let $R=1$ and $r=\epsilon_n/2$ in Theorem \ref{thm:sparse_support} and for some $b>0$, let $\delta = 1-\exp(-(b+4)n\epsilon_n^2)$ so that $\log(1-\delta) = -(b+4)n\epsilon_n^2$. Theorem \ref{thm:sparse_support} further requires that $r<cR$, which is true without loss of generality, since otherwise we can replace $r$ with $cR$ without changing our conclusions. Thus, for $N=\left\lceil\frac{-(b+4)n\epsilon_n^2}{\log(1-c(\epsilon_n/2)^k)}\right\rceil$, Theorem \ref{thm:sparse_support} implies that condition 1 is satisfied if $\log(D(\epsilon_n,\D_n)) \le  \log(N(\epsilon_n/2,\D_n))= \log N \le n\epsilon_n^2$. This is trivially satisfied if $N=1$, so assume $N>1$ such that $N\le 2\frac{-(b+4)n\epsilon_n^2}{\log(1-c(\epsilon_n/2)^k)}$.
    We have $\epsilon_n = b_4\frac{\sqrt{k\log n}}{\sqrt{n}}$ and use the inequality $(-1/\log(1-x))\le 1/x$ to derive sufficient conditions for $\log N \le n\epsilon_n^2$. Below we show that that this satisfied asymptotically in $n,k$ for any $b>0$ and $b_4> 1/\sqrt{2}$:

    \begin{align*}
n\epsilon_n^2 &\ge \log N
\impliedby
n\epsilon_n^2
    \ge
    \log\Biggl(
        2\frac{-(b+4)n\epsilon_n^2}
        {\log(1-c(\epsilon_n/2)^k)}
    \Biggr) \\
&\iff
b_4^2k\log n
    \ge
    \log\Biggl(
        2\frac{-b_4^2(b+4)k\log n}
        {\log(1-c(\epsilon_n/2)^k)}
    \Biggr) \\
&\iff
\begin{aligned}[t]
b_4^2k\log n
    \ge\;&
    \log(k\log n)
    + \log\!\bigl(2b_4^2(b+4)\bigr) 
    + \log\!\left(
        -\frac{1}{\log(1-c(\epsilon_n/2)^k)}
      \right)
\end{aligned} \\
&\impliedby
\begin{aligned}[t]
b_4^2k\log n
    \ge\;&
    \log(k\log n)
    + \log\!\bigl(2b_4^2(b+4)\bigr) 
    + \log\!\left(
        \frac{1}{c(\epsilon_n/2)^k}
      \right)
\end{aligned} \\
&\iff
b_4^2k\log n
    \ge
    \log(k\log n)
    + \log\!\left(\frac{2b_4^2(b+4)}{c}\right)
    + k\log\!\left(\frac{2}{\epsilon_n}\right) \\
&\iff
\begin{aligned}[t]
b_4^2k\log n
    \ge\;&
    \log(k\log n)
    + \log\!\left(\frac{2b_4^2(b+4)}{c}\right) + \frac{k}{2}\log n
    - \frac{k}{2}\log(k\log n) \\
    &\qquad \qquad \qquad
    + k\log\!\left(\frac{2}{b_4}\right)
\end{aligned} \\
&\iff
\begin{aligned}[t]
\log\!\left(\frac{2b_4^2(b+4)}{c}\right)   \le\; & k\left(b_4^2-\frac12\right)\log n
    + \left(\frac{k}{2}-1\right)\log(k\log n)\\
    &\qquad 
    - k\log\!\left(\frac{2}{b_4}\right).
\end{aligned}
\end{align*}

    We obtain that the condition 1 is satisfied asymptotically in $n,k$ for any $b>0$ and $b_4>1/\sqrt{2}$. Thus for $b>0, b_4>1/\sqrt{2}$ and $\epsilon_n = b_4\frac{\sqrt{k\log n}}{\sqrt{n}} $, $\exists n_3,k_3>0$ such that for $n>n_3,k>k_3$, $\log(D(\epsilon_n, \D_n)) \le n\epsilon_n^2 $ and $\mcP(\D/\D_n)\le \exp(-n\epsilon_n^2(b+4))$.

    \item \textbf{Establishing a sufficient condition for condition 3 of Lemma \ref{lem:contraction} if $b<1$:}

    The final condition of Lemma \ref{lem:contraction} may not be satisfied uniformly for every choice of $\pi_0\in\D$ from our assumptions. Instead, we calculate the probability under the prior $\mcP$ that condition 3 is satisfied for $\pi_0\sim\mcP$. We re-write condition 3 as follows:
    \begin{align*}
    1-\delta_n 
    &\le 
        Pr_{X_1,\ldots, X_n\sim \pi_0}\left(\int \prod (\pi/\pi_0)(X_i)\dd\mcP(\pi) \ge \exp(-(2+b)n\epsilon_n^2)\right)\\
    \impliedby 1-\delta_n &\le
    \begin{aligned}[t]
         Pr_{X_1,\ldots, X_n\sim \pi_0}\Biggl(\int_{\{\pi:h^2(\pi,\pi_0)\|\pi_0/\pi\|_\infty \le \epsilon_n^2\}} \prod (\pi/\pi_0)&(X_i)\dd\mcP(\pi) \\
        &\ge \exp(-(2+b)n\epsilon_n^2)\Biggr)
    \end{aligned}\\
    \impliedby 1-\delta_n &\le
    \begin{aligned}[t]
        Pr_{X_1,\ldots, X_n\sim \pi_0}\Biggl(\int_{\{\pi:h^2(\pi,\pi_0) \le \epsilon_n^2/b_{lr}\}} \prod (\pi/\pi_0)&(X_i)\dd\mcP(\pi) \\
        &\ge \exp(-(2+b)n\epsilon_n^2)\Biggr).
    \end{aligned}
\end{align*}

\noindent Note that Lemma \ref{lem:hellinger_ball} implies that:
    \begin{align*}
        \exp(-&C_{\mathrm{hellinger}}n\epsilon_n^2)\\
        &\ge 
        \begin{aligned}[t]
             Pr_{X_1,\ldots, X_n\sim \pi_0}\Biggl(\mcP(\{\pi:&h^2(\pi,\pi_0) \le \epsilon_n^2/b_{lr}\})^{-1}\\
             &\int_{\{\pi:h^2(\pi,\pi_0) \le \epsilon_n^2/b_{lr}\}} \prod (\pi/\pi_0)(X_i)\dd\mcP(\pi) \le \exp(-3n\epsilon_n^2)\Biggr)
        \end{aligned}\\
        &= 
        \begin{aligned}[t]
            Pr_{X_1,\ldots, X_n\sim \pi_0}\Biggl(\int_{\{\pi:h^2(\pi,\pi_0) \le \epsilon_n^2/b_{lr}\}} &\prod (\pi/\pi_0)(X_i)\dd\mcP(\pi)\\
            &\le \mcP(\{\pi:h^2(\pi,\pi_0) \le \epsilon_n^2/b_{lr}\})\exp(-3n\epsilon_n^2)\Biggr),
        \end{aligned}
    \end{align*}
    for the universal constant $C_{\mathrm{hellinger}}>0$. Now, if 
    \begin{align*}
        \mcP(\{\pi:h^2(\pi,\pi_0) \le \epsilon_n^2/b_{lr}\})\ge \exp(-(1-b)n\epsilon_n^2)
    \end{align*}
    and $b<1$, then we have
    \begin{align*}
        Pr_{X_1,\ldots, X_n\sim \pi_0}\Biggl(\int_{\{\pi:h^2(\pi,\pi_0) \le \epsilon_n^2/b_{lr}\}} \prod (\pi/\pi_0)(X_i)&\dd\mcP(\pi) \le \exp(-(2+b)n\epsilon_n^2)\Biggr)\\
        &\le \exp(-C_{\mathrm{hellinger}}n\epsilon_n^2),
    \end{align*}
    which satisfies condition 3 of Lemma \ref{lem:contraction} for $\delta_n = \exp(-C_{\mathrm{hellinger}}n\epsilon_n^2)$.
    
    \item \textbf{Sufficient conditions for invoking Lemma \ref{lem:contraction} occur with high probability:}
    
    \noindent Combining the above sufficient condition for Condition 3 of Lemma \ref{lem:contraction} along with our proof for Conditions 1-2, we get that for any $b\in(0,1)$, there exist $n_4,k_4,M'>0$ such that for $n>n_4,k>k_4,M>M'$,
    \begin{align*}
        \mcP(\{\pi:h^2(\pi,\pi_0) &\le \epsilon_n^2/b_{lr}\})\ge \exp(-(1-b)n\epsilon_n^2) \implies\\
        \E_{\pi_0^{\otimes n}}[\mcP(\pi:\rho(&\pi,\pi_0)>M\epsilon_n|X_1,\ldots,X_n)] \le 2\exp(-C_{\mathrm{test}}n\epsilon_n^2) + \exp(-C_{\mathrm{hellinger}}n\epsilon_n^2)\\
    &+ \exp((2+b)n\epsilon_n^2)(\exp(-(b+4)n\epsilon_n^2) + \exp(-C_{\mathrm{test}}nM^2\epsilon_n^2)).
    \end{align*}

    \noindent Further, if $C_{\mathrm{test}}M^2> (2+b)$, then there is a $b_2'>0$ such that asymptotically,
    \begin{align*}
        \mcP(\{\pi:h^2(\pi,\pi_0) \le \epsilon_n^2&/b_{lr}\})\ge \exp(-(1-b)n\epsilon_n^2) \implies\\
        &\E_{\pi_0^{\otimes n}}\left[\mcP(\pi:\rho(\pi,\pi_0)>M\epsilon_n|X_1,\ldots,X_n)\right] \le 5\exp(-b_2'n\epsilon_n^2).
    \end{align*}
    We recall that our proof for Conditions 1-2 requires that $\epsilon_n = b_4\frac{\sqrt{k\log n}}{\sqrt{n}}$ for $b_4>1/\sqrt{2}$. For any $b\in(0,1)$, as long as $C_{\mathrm{test}}M^2> (2+b)$, we can choose $b_4>M/\sqrt{2}$ which allows us to replace $\epsilon_n$ in the above equation with $\epsilon_n/M$ to obtain that there is a $b_2>0$ such that asymptotically,
    \begin{align}
        \mcP(\{\pi:h^2(\pi,\pi_0) \le \epsilon_n^2&/(M^2 b_{lr})\})\ge \exp(-(1-b)n\epsilon_n^2/M^2) \implies\nonumber\\
        &\E_{\pi_0^{\otimes n}}\left[\mcP(\pi:\rho(\pi,\pi_0)>\epsilon_n|X_1,\ldots,X_n)\right] \le 5\exp(-b_2n\epsilon_n^2).\label{eq:ball_to_exp}
    \end{align}

    We are interested in bounding the probability with which this happens, thus we want to find $Pr_{\pi_0\sim\mcP}[\mcP(\{\pi:h^2(\pi,\pi_0) \le \epsilon_n^2/(M^2 b_{lr})\})\ge \exp(-(1-b)n\epsilon_n^2/M^2))]$ for a large enough $M$. Now since $h^2(\pi,\pi') \le 2\rho_{TV}(\pi,\pi')$ by Equation \ref{eq:hell_to_tv}, $\mcP(\{\pi:\rho_{TV}(\pi,\pi_0) \le \epsilon_n^2/(2M^2 b_{lr})\} )\ge \exp(-(1-b)n\epsilon_n^2/M^2)$ implies $\mcP(\{\pi:h^2(\pi,\pi_0) \le \epsilon_n^2/(M^2 b_{lr})\})\ge \exp(-(1-b)n\epsilon_n^2/M^2))$. Thus, 
    \begin{align*}
        Pr_{\pi_0\sim\mcP}[\mcP&(\{\pi:h^2(\pi,\pi_0) \le \epsilon_n^2/(M^2 b_{lr})\})\ge \exp(-(1-b)n\epsilon_n^2/M^2))]\\
        &\ge Pr_{\pi_0\sim\mcP}[\mcP(\{\pi:\rho_{TV}(\pi,\pi_0) \le \epsilon_n^2/(2M^2 b_{lr})\} )\ge \exp(-(1-b)n\epsilon_n^2/M^2))],
    \end{align*}
    and we can focus on the latter expression.

    \noindent To bound the above expression, we use a similar argument as used in the proof of conditions 1-2 and consider sets $\A_1,\A_2,\ldots$ defined similarly. Let $\A_0=\D$ and $\A_{i+1}\subseteq \D/(\bigcup_{j=1}^i \A_j) $ be the highest measure set that satisfies
    \begin{align*}
        \diam(\A_{i+1}) \le \epsilon_n^2/(2M^2 b_{lr})\diam(\A_i) \le \epsilon_n^2/(2M^2 b_{lr}).
    \end{align*}
    By definition, $\mcP(\A_1) \ge \mcP(\A_2) \ldots $ and let $\mcP(\A_{N^\delta}) \ge \delta c(\epsilon_n^2/(2M^2 b_{lr}))^k \ge \mcP(\A_{N^\delta+1})$ for some $\delta\in (0,1)$. Note that $N^\delta\ge 1$ since $\mcP(\A_1) \ge c(\epsilon^2/(2M^2 b_{lr}))^k$ and further that such an $N^\delta$ exists due to the disjointness of the sets $\A_i$ and since the total probability is 1. Now, $\mcP(\A_{N^\delta+1}) \ge c(\epsilon_n^2/(2M^2 b_{lr}))^k(1-\sum_{i=1}^{N^\delta}\mcP(\A_i)) $ by Definition \ref{def:sparse_dim}. We can thus conclude:
    \begin{align*}
        c(\epsilon_n^2/(M^2 b_{lr}))^k\left(1-\sum_{i=1}^{N^\delta}\mcP(\A_i)\right) \le \delta c(\epsilon_n^2/(2M^2 b_{lr}))^k \implies \sum_{i=1}^{N^\delta}\mcP(\A_i) &\ge 1-\delta.
    \end{align*}

    Now for every $\pi \in \A_i$ for $i\le N^\delta$, since the diameter of $\A_i$ is $\le \epsilon_n^2/(2M^2 b_{lr})$, 
    \begin{align*}
        \mcP(B(\epsilon_n^2/(2M^2 b_{lr}),\pi)) \ge \mcP(\A_i) \ge \delta c(\epsilon_n^2/(2M^2 b_{lr}))^k,
    \end{align*}
    where we define $B(r,\pi) = \{\pi':\rho(\pi',\pi)\le r\}$ for ease of notation. Thus define $\D_{\delta} = \bigcup_{i=1}^{N^\delta}\A_i $ to obtain that for every $\pi\in\D_\delta$,
    \begin{align*}
        \mcP(B(\epsilon_n^2/(2M^2 b_{lr}),\pi)) \ge \delta c(\epsilon^2/(2M^2 b_{lr}))^k.
    \end{align*}
    Thus, $Pr_{\pi_0\sim\mcP}(\mcP(B(\epsilon_n^2/(2M^2 b_{lr}),\pi_0)) \ge \delta c(\epsilon_n^2/(2M^2 b_{lr}))^k) \ge \mcP(\D_\delta) = 1-\delta$. We pick a $\delta$ such that:
    \begin{align*}
        \delta c(\epsilon_n^2/(2M^2 b_{lr}))^k &= \exp(-(1-b)n\epsilon_n^2/M^2) = \exp(-b_4^2(1-b)k\log n/M^2)\\
        \implies \log\delta + \log(c) &= -\frac{b_4^2(1-b)}{M^2}k\log n - k\log \epsilon_n^2 + k\log(2M^2 b_{lr})\\
        &= -k\left(\frac{b_4^2(1-b)}{M^2}\log n - \log (2M^2 b_{lr})\right) - k\log (b_4^2k\log n) + k\log n\\
        &\le -k\left(\frac{b_4^2(1-b)}{M^2}\log n - \log n - \log (2M^2 b_{lr})\right).
    \end{align*}

    For any $b\in(0,1)$, we already have the constraints that $C_{\mathrm{test}}M^2> (2+b)$ and $b_4>M/\sqrt{2}$. We additionally impose $\frac{b_4^2(1-b)}{M^2} > 1$ such that $\delta$ asymptotically decreases with $n$. Thus, for any $b_{lr}>0, b\in(0,1)$, there exist $n_4=\omega(b_{lr}), k_4>0$ and a constant $b_3>0$ such that for all $n>n_4,k>k_4$, $\delta \le \exp(-b_3b_4^2k\log n) = \exp(-b_3n\epsilon_n^2)$. Thus for $n>n_4,k>k_4$, using Equation \ref{eq:ball_to_exp},
    \begin{align*}
        &Pr_{\pi_0\sim\mcP}(\mcP(B(\epsilon_n^2/(2M^2 b_{lr}),\pi_0)) \ge \exp(-(1-b)n\epsilon_n^2/M^2))  \ge 1-\exp(-b_3n\epsilon_n^2)\\
        &\implies
        \begin{aligned}[t]
            Pr_{\pi_0\sim\mcP}(\E_{\pi_0^{\otimes n}}\left[\mcP(\pi:\rho(\pi,\pi_0)>\epsilon_n|X_1,\ldots,X_n)\right] \le 5\exp(-&b_2n\epsilon_n^2))\\
            &\ge 1-\exp(-b_3n\epsilon_n^2).
        \end{aligned}
    \end{align*}

    We now recall from Equation \ref{eq:exp_break} that we wanted to establish the above condition for constants $b_2,b_3,b_4$ (that do not depend on $n,k$) and $\epsilon_n = b_4\frac{\sqrt{k\log n}}{\sqrt{n}}$. We have thus established all requirements to state our final bound,
    \begin{align*}
        \E_{\pi_0\sim\mcP}[\E_{X_1,\ldots,X_n\sim\pi_0}[\rho_{TV}(\pi_0, \hat{\pi}_n)]] = O\left(\frac{\sqrt{k\log n}}{\sqrt{n}}\right).
    \end{align*}
    
    \end{enumerate}

\end{proof}
\end{appendix}

\end{document}